\documentclass{article}

\usepackage{microtype}
\usepackage{graphicx, multirow}
\usepackage{subfigure}
\usepackage{booktabs} 

\usepackage{amsmath, amssymb, amsfonts, amsthm, bbm}
\usepackage{mathtools,commath}
\usepackage{xcolor}
\usepackage{authblk}

\usepackage{pifont}
\newcommand{\cmark}{\ding{51}}%
\newcommand{\xmark}{\ding{55}}%

\usepackage{hyperref}

\newtheorem{theorem}{Theorem}
\newtheorem{proposition}{Proposition}
\newtheorem{lemma}[theorem]{Lemma}
\newtheorem{corollary}{Corollary}
\newtheorem{definition}{Definition}
\newtheorem{example}{Example}
\newtheorem{remark}{Remark}

\newcommand{\cA}{\mathcal{A}}
\newcommand{\cE}{\mathcal{E}}
\newcommand{\cF}{\mathcal{F}}
\newcommand{\cG}{\mathcal{G}}
\newcommand{\cH}{\mathcal{H}}
\newcommand{\cL}{\mathcal{L}}

\newcommand{\cN}{\mathcal{N}}
\newcommand{\cP}{\mathcal{P}}
\newcommand{\cS}{\mathcal{S}}

\newcommand{\cW}{\mathcal{W}}
\newcommand{\cX}{\mathcal{X}}
\newcommand{\cY}{\mathcal{Y}}
\newcommand{\cZ}{\mathcal{Z}}

\newcommand{\fM}{\mathfrak{M}}

\newcommand{\bE}{\mathbb{E}}
\newcommand{\bP}{\mathbb{P}}
\newcommand{\bQ}{\mathbb{Q}}
\newcommand{\bR}{\mathbb{R}}

\newcommand{\mLip}{\mathrm{Lip}}

\allowdisplaybreaks

\usepackage{hyperref}

\usepackage[accepted]{icml2020}

\icmltitlerunning{Principled learning method for WDRO with local perturbations}

\begin{document}

\twocolumn[
\icmltitle{Principled Learning Method for Wasserstein Distributionally Robust Optimization with Local Perturbations}

\begin{icmlauthorlist}
\icmlauthor{Yongchan Kwon}{stan}
\icmlauthor{Wonyoung Kim}{snu}
\icmlauthor{Joong-Ho Won}{snu}
\icmlauthor{Myunghee Cho Paik}{snu}
\end{icmlauthorlist}

\icmlaffiliation{stan}{Department of Biomedical Data Science, Stanford University}
\icmlaffiliation{snu}{Department of Statistics, Seoul National University}

\icmlcorrespondingauthor{Myunghee Cho Paik}{myungheechopaik@snu.ac.kr}

\icmlkeywords{Wasserstein distributionally robust optimization, robust learning}

\vskip 0.3in
]

\printAffiliationsAndNotice{This work was done while Yongchan Kwon was at Seoul National University.}  

\begin{abstract}
Wasserstein distributionally robust optimization (WDRO) attempts to learn a model that minimizes the local worst-case risk in the vicinity of the empirical data distribution defined by Wasserstein ball.
While WDRO has received attention as a promising tool for inference since its introduction, its theoretical understanding has not been fully matured.  
\citet{gao2017} proposed a minimizer based on a tractable approximation of the local worst-case risk, but without showing risk consistency.
In this paper, we propose a minimizer based on a novel approximation theorem and provide the corresponding risk consistency results.
Furthermore, we develop WDRO inference for locally perturbed data that include the Mixup \citep{zhang2017} as a special case.
cNumerical experiments demonstrate robustness of the proposed method using image classification datasets.
Our results show that the proposed method achieves significantly higher accuracy than baseline models on contaminated datasets.
\end{abstract}

\section{Introduction}
\label{s:introduction}
Statistical learning problems can be generally formulated as an optimization problem of the form
\begin{align}
\inf_{h \in \cH} R(\bP_{\mathrm{data}}, h),
\label{eq:risk_min}
\end{align}
where $\bP_{\mathrm{data}}$ is the true data distribution, $\cH$ is a set of losses, and $R(\bQ, h) := \int h(\zeta) d\bQ (\zeta)$ is the risk, or the expected value of a loss $h$ with respect to a probability measure $\bQ$.
In real-world applications, the $\bP_{\mathrm{data}}$ is usually unknown, so the computation of the risk in \eqref{eq:risk_min} is impossible. 
We instead observe a set $\cZ_n =\{z_1, \dots, z_n\}$ of independent and identically distributed samples from $\bP_{\mathrm{data}}$.
Using the dataset $\cZ_n$, we solve the empirical risk minimization (ERM) problem
\begin{align}
\inf_{h \in \cH} R(\bP_{n}, h) = \inf_{h \in \cH} \frac{1}{n} \sum_{i=1} ^n h(z_i),
\label{eq:emp_risk_min}
\end{align}
where $\bP_n := n^{-1} \sum_{i=1} ^n \delta_{z_i}$ is the empirical data distribution and $\delta_z$ is the Dirac delta distribution concentrating unit mass at $z$.

ERM provides a practical framework for learning models by replacing $\bP_{\mathrm{data}}$ in \eqref{eq:risk_min} with $\bP_n$ \citep{vapnik1999}.
However, this replacement often yields poor risk estimation, and thus a solution to \eqref{eq:emp_risk_min} can have a small training error but a large test error.
This phenomenon is well known as overfitting, and to avoid this, a great number of regularization methods have been proposed: penalty-based methods \citep{tibshirani1996, fan2001, buhlmann2011}, data augmentations \citep{zhang2017, cubuk2019}, dropout \citep{wager2013, srivastava2014}, and early stopping \citep{yao2007}, to name a few.

As an alternative approach to prevent overfitting, we consider Wasserstein distributionally robust optimization (WDRO) \citep{shafieezadeh2015, sinha2017, blanchet2019robust}.
The goal of WDRO is to learn a model that minimizes the local worst-case risk in the vicinity of the empirical data distribution defined by a Wasserstein ball. 
To be specific, let $\fM_{\alpha_n,p} (\bP_n)$ be a set of probability measures whose $p$-Wasserstein metric from $\bP_n$ is less than $\alpha_n >0$. 
The local worst-case risk is defined to be the supremum of the risk over the $p$-Wasserstein ball.
Then WDRO is formulated as follows. 
\begin{align*}
\inf_{h \in \cH} \sup_{\bQ \in \fM_{\alpha_n,p} (\bP_n)} R(\bQ,h).
\end{align*}
Detailed definitions are available in Section \ref{s:preliminaries}.

\begin{table*}[t]
\caption{A summary of contributions of \citet{gao2017}, \citet{lee2018} and ours. Mark \lq{}\cmark\rq{} indicates that the corresponding work has a contribution to the item; mark \lq{}\xmark\rq{} means otherwise. We denote a sample space by $\cZ \subseteq \bR^{d}$.}
\label{t:summary_contribution}
\vskip 0.15in
\begin{center}
\begin{small}
\begin{sc}
\begin{tabular}{llll}
\toprule
& Approximation & Risk consistency & Perturbation \\
\midrule
\citet{gao2017}  & \cmark ($\cZ = \bR^d$)          &         \xmark         &       \xmark      \\
\citet{lee2018}  &   \xmark            & \cmark     ($\cZ$ is bounded)          &     \xmark        \\
Ours & \cmark  ($\cZ$ is bounded)       & \cmark   ($\cZ$ is bounded; Section \ref{s:wdro})            & \cmark  (Section \ref{s:wdro_perturbed})    \\ 
\bottomrule
\end{tabular}
\end{sc}
\end{small}
\end{center}
\vskip -0.1in
\end{table*}

By the design of the local worst-case risk, a solution to WDRO can avoid overfitting to $\bP_n$ and learn a robust model against local perturbations.
However, exact computation of the local worst-case risk is intractable except for few simple settings because it is difficult (i) to evaluate an exact risk with respect to a probability measure in the Wasserstein ball and (ii) to find the supremum of the risk among infinitely many probability distributions.
\citet{gao2017} obtained an approximation formula for the local worst-case risk and proposed to minimize this surrogate objective, but did not study risk consistency of the minimizer.
Under a different assumption on the sample space, \citet{lee2018} proved that a minimizer of the exact local worst-case risk possesses risk consistency.
However, finding such a minimizer is difficult due to the intractability of the local worst-case risk.
To the best of our knowledge, there is no known risk consistency result for tractable approximate optimizers.

In this paper, we propose a minimizer based on a novel approximation theorem and provide corresponding risk consistency results. 
In Section \ref{s:wdro}, we present a new approximation to the local worst-case risk using gradient penalty assuming that a loss is differentiable and its gradient has a H{\"o}lder continuous (Theorem \ref{thm:wdro_approx_sharp_bound}).
We show that a minimizer of the approximate worst-case risk is consistent in that the risk (\textit{resp.} the worst-case risk) converges to the optimal risk (\textit{resp.} the optimal worst-case risk) (Theorems \ref{thm:wdro_excess_worst_risk} and \ref{thm:wdro_excess_risk}).
Our results show that the proposed minimizer can have the same risk optimality as a minimizer of the exact local worst-case risk attains.
In Section \ref{s:wdro_perturbed}, we study WDRO inference when data are locally perturbed. 
We define locally perturbed data distributions and describe examples such as the Mixup \citep{zhang2017} and the adversarial training \citep{goodfellow2014}.
We show that our approximation and risk consistency results naturally extend to the cases when data are locally perturbed (Theorems \ref{thm:wdro_perturb_approx_sharp_bound}, \ref{thm:wdro_perturb_excess_worst_risk}, and \ref{thm:wdro_perturb_excess_risk}).
Such theoretical results provide principled ways to use a group of data augmentation including the Mixup.
Numerical experiments demonstrate robustness of the proposed method using image classification datasets.
Our experiment results show that the proposed method produces a robust model that achieves significantly higher accuracy than baseline models on contaminated datasets.

A summary of our contributions in relation to \citet{gao2017} and \citet{lee2018} is shown in Table \ref{t:summary_contribution}.
Proofs are available in the Supplementary Material.

\subsection{Related works}
Distributionally robust optimization (DRO) provides a general learning framework of the local worst-case risk minimization.
Here, the local worst-case risk is defined as the supremum of the risk in the vicinity of the empirical data distribution, called the ambiguity set.
The ambiguity set is often designed as a neighborhood of $\bP_n$ and the closeness of two measures is evaluated by $\phi$-divergences or probability metrics.
Note that WDRO is a special case of DRO when the ambiguity set is designed via the Wasserstein metric.
Other examples incorporate the $\phi$-divergence \citep{ben2013, hu2018, namkoong2017, ghosh2019} and the maximum mean discrepancy \citep{staib2019}.
We refer to \citet{rahimian2019} for a complementary literature review of DRO.

Another related field of this work is data augmentation.
Data augmentation has recently emerged as a key technique to improve empirical performance in the field of machine learning \citep{cubuk2019, lim2019}.
For example, Mixup and its variants have led remarkable generalization ability in supervised and semi-supervised learning tasks \citep{zhang2017, verma2019, berthelot2019}.
However, most data augmentations are based on heuristics, and their theoretical bases are limited to account for current successes. 
In this work, we develop WDRO inference for a group of data augmentations that generate a new data distribution near the original data distribution. 

\subsection{Notation}
For a sequence $(a_n)$ of positive constants and a sequence $(b_n)$ of real numbers, $b_n = O(a_n)$ indicates that there exists constants $C, n_0 \in \mathbb{N}$ such that $ | b_n | \leq C a_n$ for all $n \geq n_0$. 
For a random sequence $(B_n)$, $B_n = O_p(a_n)$ indicates that for any $\varepsilon > 0$, there exists constants $C, n_0 \in \mathbb{N}$ such that $P( |B_n| > C a_n) < \varepsilon$ for all $n \geq n_0$.
For a $p \in [1, \infty]$, we denote its H{\"o}lder conjugate by $p^* := (1- 1/p)^{-1}$.
Here, we use the conventions $1/\infty = 0$ and $1/0 =\infty$.
For $a,b \in \bR$, we use $a \vee b$ to denote the maximum between $a$ and $b$.
For $n \in \mathbb{N}$, we use $[n]$ to denote a set of integers $\{ 1, \dots, n\}$.
A set of all Borel probability measures defined on a set $\cS$ is denoted by $\cP(\cS)$.
We denote a sample space by $\cZ \subseteq \mathbb{R}^d$ and a norm on $\cZ$ by $\norm{\cdot}$ and the true data distribution by $\bP_{\mathrm{data}} \in \cP(\cZ)$. 

\section{Preliminaries}
\label{s:preliminaries}
The goal of this section is to review existing works on WDRO.
As discussed in Section \ref{s:introduction}, the main objective of WDRO is to learn a model that minimizes the local worst-case risk over some Wasserstein ball. 
Formally, for sets $S$ and $\tilde{S}$, we denote the push-forward measure of $\mu \in \cP(S)$ through a map $T: S \to \tilde{S}$ by $T\# \mu \in \cP(\tilde{S})$.
The definitions of the $p$-Wasserstein metric and the $p$-Wasserstein ball are as follows.
\begin{definition}[$p$-Wasserstein metric and $p$-Wasserstein ball]
For $p \in [1, \infty)$ and $\nu, \mu \in \cP(\cZ)$, the $p$-Wasserstein metric between $\nu$ and $\mu$ is defined as
\begin{align*}
&\cW_p(\nu,\mu) := \Big( \underset{\rho \in J(\nu,\mu)}{\inf} \Big\{ \int_{\cZ \times \cZ} \norm{\zeta - \tilde{\zeta}}^p d\rho(\zeta, \tilde{\zeta}) \Big\} \Big)^{1/p},
\end{align*}
where $J(\nu,\mu) := \{ \rho \in \cP(\cZ \times \cZ) \mid \pi_1 \# \rho =\nu, \pi_2 \# \rho = \mu \}$, $\pi_i: \cZ \times \cZ \to \cZ$ is the canonical projection defined by $\pi_i(\zeta_1, \zeta_2)=\zeta_i$ for $i=1,2$.
For $\alpha >0$, the $p$-Wasserstein ball centered at $\bP \in \cP(\cZ)$ with radius $\alpha$ is defined as
\begin{align*}
\fM_{\alpha,p} (\bP) &:= \{ \bQ \in \cP(\cZ) : \cW_p (\bQ, \bP) \leq \alpha \}.
\end{align*}
\end{definition}
Throughout this paper, we denote the radius of the Wasserstein ball by $\alpha_n$ when the sample size is $n$.
With above definitions, the WDRO problem is to minimize the local worst-case risk 
\begin{align}
R_{\alpha_n,p} ^{\mathrm{worst}}(\bP_n, h) := \sup_{\bQ \in \fM_{\alpha_n,p} (\bP_n)} R(\bQ, h).
\label{eq:wdro_risk}
\end{align}
The local worst-case risk \eqref{eq:wdro_risk} involves the supremum operator over the Wasserstein ball $\fM_{\alpha_n,p} (\bP_n)$, which is a set of infinitely many probability distributions.
Therefore, the exact computation of \eqref{eq:wdro_risk} is intractable in many cases. 

A standard method to handle the intractability is to reformulate \eqref{eq:wdro_risk} by using either a primal-dual pair of infinite-dimensional linear programs \citep{esfahani2018} or first-order optimality conditions of the dual \citep{gao2016}.
For the latter, let $\kappa_h = \limsup_{\norm{\zeta-\tilde{\zeta}} \to \infty} (h(\zeta)-h(\tilde{\zeta})) / \norm{\zeta-\tilde{\zeta}}^p$ if $\cZ$ is unbounded, and zero otherwise.
\citet[Corollary 2]{gao2016} showed that if $\kappa_h < \infty$, then
\begin{align}
& R_{\alpha_n,p} ^{\mathrm{worst}}(\bP_n, h) \notag \\
&= \min_{\lambda \geq 0} \Big\{ \lambda \alpha_n^p + \frac{1}{n} \sum_{i=1} ^n \sup_{z \in \cZ } \left\{ h(z) - \lambda \norm{z - z_i}^p \right\} \Big\}.
\label{eq:dro_duality}
\end{align}
Similar results are obtained in the literature \citep{gao2017dependence, blanchet2019}.

Based on the reformulation \eqref{eq:dro_duality}, relationships between WDRO and penalty-based methods have been investigated in supervised learning settings.
For example, \citet{shafieezadeh2015} and \citet{blanchet2019robust} studied classification settings and \citet{chen2018} considered regression settings.
Although the relationships provide a way to understand WDRO, most existing results focus on linear hypotheses only.
Recently, WDRO with nonlinear hypotheses has been developed. 
\citet{shafieezadeh2019} showed that \eqref{eq:wdro_risk} has the form of a penalized empirical risk when a loss is Lipschitz continuous and a hypothesis is an element of a reproducing kernel Hilbert space.

In general statistical learning problems, \citet{gao2017} established a relationship between WDRO and penalty-based methods.
They obtained a penalized empirical risk and showed that it approximates  \eqref{eq:wdro_risk} when a loss is smooth and $\cZ = \bR^{d}$.
Although a minimizer of the suggested approximation gives a practical solution for WDRO, its risk consistency has not been studied.

As for the risk consistency, \citet{lee2018} showed that a minimizer of \eqref{eq:wdro_risk} has a vanishing excess worst-case risk bound when $\cH$ is a set of Lipschitz continuous losses.
More specifically, let $\hat{h}_{\alpha_n, p} ^{\mathrm{worst}}  = \text{argmin}_{h \in \cH} R_{\alpha_n,p} ^{\mathrm{worst}}(\bP_n, h)$ and $R_{\alpha_n,p} ^{\mathrm{worst}}(\bP_{\mathrm{data}}, h) $ $ = \sup_{\bQ \in \fM_{\alpha_n,p} (\bP_{\mathrm{data}})} R(\bQ, h)$. 
For bounded $\cZ$, \citet[Theorem 2]{lee2018} showed the following risk consistency result. 
\begin{align}
&\cE_{\alpha_n,p} ^{\mathrm{worst}}(\hat{h}_{\alpha_n, p} ^{\mathrm{worst}} )= O_p ( n^{-1/2}( \mathfrak{C}(\cH) \vee \alpha_n ^{1-p}) ),
\label{eq:wdro_risk_consistency}
\end{align}
where
\begin{align*}
\cE_{\alpha_n,p} ^{\mathrm{worst}}(g) := R_{\alpha_n,p} ^{\mathrm{worst}}(\bP_{\mathrm{data}}, g) - \inf_{h \in \cH} R_{\alpha_n,p} ^{\mathrm{worst}}(\bP_{\mathrm{data}}, h)
\end{align*}
is the excess worst-case risk of $g \in \cH$, $\mathfrak{C}(\cS):= \int_{0} ^{\infty} \sqrt{\log \cN(u, \cS, \norm{\cdot}_{\infty}) } du$ is the entropy integral of a set $\cS$, and $\cN(u, \cS, \norm{\cdot}_{\infty})$ denotes the $u$-covering number of a set $\cS$ with respect to the uniform norm $\norm{\cdot}_{\infty}$ \citep[Definition 9.2]{gyorfi2006}.
The result \eqref{eq:wdro_risk_consistency} explains asymptotic behaviors of the WDRO solution $\hat{h}_{\alpha_n, p} ^{\mathrm{worst}}$, but as mentioned before, exact computation of \eqref{eq:wdro_risk} is intractable except for few simple cases.

It is noteworthy that \citet{gao2017} and \citet{lee2018} have conflicting assumptions on $\cZ$, so the results of \citet{lee2018} cannot be used to show risk consistency of the minimizer by \citet{gao2017}. 

\section{Tractable WDRO and risk consistency}
\label{s:wdro}
In this section, we build a principled and tractable learning method for WDRO.
In Section \ref{s:approx}, we propose an approximation of the local worst-case risk that can be easily evaluated by off-the-shelf gradient methods and software.
In Section \ref{s:risk_consistnecy}, we provide asymptotic results: a minimizer of the approximate risk is consistent in that the risk (\textit{resp.} the worst-case risk) converges to the optimal risk (\textit{resp.} the optimal worst-case risk).

\subsection{Approximation to the local worst-case risk}
\label{s:approx}
For a Lipschitz continuous loss $h: \cZ \to \bR$, \citet[Proposition 1]{lee2018} showed that
\begin{align}
\left| R(\bP_n, h) - R_{\alpha_n,p} ^{\mathrm{worst}}(\bP_n, h) \right| = O(\alpha_n).
\label{eq:wdro_approx_bound}
\end{align}
An equivalent result is obtained by \citet[Theorem 5]{kuhn2019}.
We show that a faster approximation is possible if a loss $h$ is differentiable and its gradient is H{\"o}lder continuous.
To begin, we define some notations. 
For $r \in [1,\infty)$, a probability measure $\bQ \in \cP(\bR^d)$, and a function $g:\cZ \to \bR^d$, we denote a function norm by
\begin{align*}
\norm{g}_{\bQ, r} := ( \int \norm{g(z)}_* ^r d\bQ(z) )^{1/r},    
\end{align*}
where $\norm{\cdot}_{*}$ is the dual norm of $\norm{\cdot}$.
For a constant $C_{\mathrm{H}}>0$ and $k \in (0,1]$, a function $g:\cZ \to \bR^d$ is said to be $(C_{\mathrm{H}}, k)$-H{\"o}lder continuous if 
\begin{align*}
\norm{ g (z) - g(\tilde{z})}_{*} \leq C_{\mathrm{H}} \norm{z- \tilde{z}} ^{k}, \quad \forall z, \tilde{z} \in \cZ.
\end{align*}
Let $\mathrm{Conv}(\cZ)$ be the convex hull of $\cZ$ and $\bE_{\mathrm{data}}(g) := \int_{\cZ} g(z) d \bP_{\mathrm{data}} (z)$ for a function $g:\cZ \to \bR$.

\begin{theorem}[Approximation to local worst-case risk]
Let $(\alpha_n)$ be a sequence of positive numbers converging to zero and $\cZ$ be an open and bounded subset of $\bR^d$.
For constants $C_{\mathrm{H}}, C_{\nabla} >0$ and $k \in (0,1]$, assume that a loss $h: \mathrm{Conv}(\cZ) \to \bR$ is differentiable, its gradient $\nabla_z h(z)$ is $(C_{\mathrm{H}}, k)$-H{\"o}lder continuous, and $\bE_{\mathrm{data}}(\norm{\nabla_z h}_*) \geq C_{\nabla}$.
Then, for $p \in (1+k, \infty)$, the following holds.
\begin{align*}
&\left| R(\bP_n, h) + \alpha_n \norm{\nabla_z h}_{\bP_n, p^*} -R_{\alpha_n,p} ^{\mathrm{worst}}(\bP_n, h) \right| \\
&=O_p (\alpha_n ^{1+k}).
\end{align*}
\label{thm:wdro_approx_sharp_bound}
\end{theorem}
\begin{remark}
Theorem \ref{thm:wdro_approx_sharp_bound} establishes an asymptotic equivalence between WDRO and penalty-based methods.
Compared to \eqref{eq:wdro_approx_bound}, Theorem \ref{thm:wdro_approx_sharp_bound} provides a sharper approximation to the local worst-case risk.
\citet[Theorem 2]{gao2017} obtained a similar result when $\cZ =\bR^d$, yet our boundedness assumption on $\cZ$ is reasonable in a sense that real computers store data in a finite number of states.
For example, a $d$-dimensional gray scale image datum is stored as a $d$-dimensional vector having integer values range from 0 to 255.
\end{remark}

\begin{remark}
The assumption $\bE_{\mathrm{data}}(\norm{\nabla_z h}_*) \geq C_{\nabla}$ in Theorem \ref{thm:wdro_approx_sharp_bound} holds as long as there exist positive constants $C_{\nabla,1}$ and $C_{\nabla,2}$ such that $P(\norm{\nabla_z h}_* \geq C_{\nabla,1}) \geq C_{\nabla,2}$. 
Note that by the Markov\rq{}s inequality, $\bE_{\mathrm{data}}(\norm{\nabla_z h}_*) \geq C_{\nabla,1}C_{\nabla,2}$. 
Hence, unless $h$ is a constant function, $\norm{\nabla_z h}_*$ is strictly greater than zero and existence of $C_{\nabla}$ is guaranteed. 
\end{remark}

Based on Theorem \ref{thm:wdro_approx_sharp_bound}, for a vanishing sequence $(\alpha_n)$, we propose to minimize the following surrogate objective:
\begin{align}
R_{\alpha_n, p} ^{\mathrm{prop}} (\bP_{n}, h) := R(\bP_n, h) + \alpha_n \norm{\nabla_z h}_{\bP_n, p^*}.
\label{eq:prop_bound}
\end{align}
In the sequel, we denote a minimizer of the objective function \eqref{eq:prop_bound} by $\hat{h}_{\alpha_n, p} ^{\mathrm{prop}}$, \textit{i.e.}, $\hat{h}_{\alpha_n, p} ^{\mathrm{prop}} = \mathrm{argmin}_{h \in \cH} R_{\alpha_n, p} ^{\mathrm{prop}} (\bP_{n}, h)$. 
In contrast to the intractability of \eqref{eq:wdro_risk}, the approximate risk \eqref{eq:prop_bound} can be easily minimized by off-the-shelf gradient methods and software.

\subsection{Risk consistency of the proposed estimator}
\label{s:risk_consistnecy}
We then study the excess worst-case risk bound of $\hat{h}_{\alpha_n, p} ^{\mathrm{prop}}$.
To begin with, for a Lipschitz continuous function $h$, we denote the smallest Lipschitz constant of $h$ by $\mLip(h)$. 

\begin{theorem}[Excess worst-case risk bound]
Let $(\alpha_n)$ be a sequence of positive numbers converging to zero and $\cZ$ be an open and bounded subset of $\bR^d$.
For constants $C_{\mathrm{H}}, C_{\nabla}, L >0$ and $k \in (0,1]$, assume that $\cH$ is a uniformly bounded set of differentiable functions $h: \mathrm{Conv}(\cZ) $ $\to \mathbb{R}$ such that its gradient $\nabla_z h$ is $(C_{\mathrm{H}}, k)$-H{\"o}lder continuous, $\bE_{\mathrm{data}}(\norm{\nabla_z h}_*) \geq C_{\nabla}$, and $\mLip(h) \leq L$. 
Then, for $p \in (1+k, \infty)$, the following holds.
\begin{align*}
\cE_{\alpha_n,p} ^{\mathrm{worst}}( \hat{h}_{\alpha_n,p} ^{\mathrm{prop}}) = O_p \left( \frac{\mathfrak{C}(\cH) \vee \alpha_n ^{1-p}}{\sqrt{n}} \vee \log(n) \alpha_n ^{1+k} \right).
\end{align*}
\label{thm:wdro_excess_worst_risk}
\end{theorem}
Compared to the bound \eqref{eq:wdro_risk_consistency} by \citet{lee2018}, the risk bound of the proposed method in Theorem \ref{thm:wdro_excess_worst_risk} has the additional term $\log(n) \alpha_n ^{1+k}$. 
This additional error is a payoff for the approximation \eqref{eq:prop_bound}, and it is asymptotically negligible when $\alpha_n ^{-(p+k)} \geq O(n^{1/2}\log (n))$.
Thus the proposed minimizer can have the same risk optimality as $\hat{h}_{\alpha_n, p} ^{\mathrm{worst}}$ achieves.

Next, we analyze the excess risk bound of $\hat{h}_{\alpha_n, p} ^{\mathrm{prop}}$.
Recall that the Rademacher complexity of a set $\cS$ is defined as $\mathfrak{R}_{n}(\cS) := \bE_{\mathrm{data}} \bE_{\sigma} ( \frac{1}{n} \sup_{s \in \cS} | \sum_{i=1} ^{n} \sigma_i s(Z_i) |  )$ where $\{\sigma_i\}_{i=1} ^{n}$ is a set of independent Rademacher random variables taking $1$ or $-1$ with probability $0.5$ each, and $\bE_{\sigma}(\cdot)$ is the expectation operator over the Rademacher random variables \citep{bartlett2002}.
We denote the excess risk of $g\in \cH$ by $\cE(g) := R(\bP_{\mathrm{data}}, g) - \inf_{h \in \cH} R(\bP_{\mathrm{data}}, h).$
\begin{theorem}[Excess risk bound]
Under the same assumptions as Theorem \ref{thm:wdro_excess_worst_risk}, the following holds.
\begin{align*}
\cE( \hat{h}_{\alpha_n,p} ^{\mathrm{prop}} ) = O_p( \mathfrak{R}_n (\mathcal{H}) \vee n^{-1/2} \vee \alpha_n \vee \log(n) \alpha_n ^{1+k}  ).
\end{align*} 
\label{thm:wdro_excess_risk}
\end{theorem}
Suppose $\alpha_n =n^{-\epsilon}$ for some $\epsilon >0$. 
Then, for a large enough $n$, we have $\alpha_n \vee \log(n) \alpha_n ^{1+k}  = \alpha_n$ and the excess risk bound is $ O_p( \mathfrak{R}_n (\mathcal{H}) \vee n^{-1/2} \vee \alpha_n )$.
Considering the fact that the excess risk bound of the ERM solution is $O_p(\mathfrak{R}_n (\mathcal{H}) \vee n^{-1/2})$ \citep[Theorem 11.3]{mohri2018}, the result of Theorem \ref{thm:wdro_excess_risk} sounds pessimistic, especially when $\epsilon \leq 1/2$.
However, Theorem \ref{thm:wdro_excess_risk} is in fact sensible in that $\hat{h}_{\alpha_n,p} ^{\mathrm{prop}}$ optimizes the local worst-case risk $R_{\alpha_n,p} ^{\mathrm{worst}}(\bP_n, h)$, not the risk $R(\bP_n, h)$.
That means, gaining robustness necessarily leads to losing the accuracy of the prediction model. 
Interested readers in the trade-off between accuracy and robustness are referred to \citet{zhang2019}.

\subsection{Example bounds}
\label{s:example_bounds}
We now provide an example showing the use of Theorems \ref{thm:wdro_excess_worst_risk} and \ref{thm:wdro_excess_risk} in binary classification settings.
To begin, we denote a solution of \eqref{eq:emp_risk_min} by $\hat{h}_{n} ^{\mathrm{ERM}}$.
Let $\cX \subseteq [-1,1]^{d-1}$ and $\cY = \{ \pm 1\}$ be open sets with respect to the $\ell_2$-norm and the discrete norm $I(\cdot \neq 0)$, respectively.
We set $\cZ = \cX \times \cY$ and $\norm{(x,y)} = \norm{x}_{2} + 4I(y \neq 0)$.
\begin{corollary}[Informal]
Let $\cF$ be a set of sparse deep neural networks and $\cH =\{ h(x,y) \mid h(x,y) = \log(1+\exp(-yf(x)))$ for $f \in \cF \}$. 
Then the excess worst-case risks of $\hat{h}_{\alpha_n,p} ^{\mathrm{prop}}$ and $\hat{h}_{n} ^{\mathrm{ERM}}$ are
\begin{align*}
\cE_{\alpha_n,p} ^{\mathrm{worst}}( \hat{h}_{\alpha_n,p} ^{\mathrm{prop}})&= O_p( n^{-1/2} \alpha_n ^{1-p} \vee \log(n) \alpha_n ^{1+k} ), \\
\cE_{\alpha_n,p} ^{\mathrm{worst}}( \hat{h}_{n} ^{\mathrm{ERM}}) &= O_p( n^{-1/2} \vee \alpha_n ).
\end{align*}
Furthermore, the excess risks of $\hat{h}_{\alpha_n,p} ^{\mathrm{prop}}$ and $\hat{h}_{n} ^{\mathrm{ERM}}$ are
\begin{align*}
\cE(\hat{h}_{\alpha_n,p} ^{\mathrm{prop}})&= O_p( n^{-1/2} \vee \alpha_n \vee \log(n) \alpha_n ^{1+k}  ), \\
\cE( \hat{h}_{n} ^{\mathrm{ERM}}) &= O_p( n^{-1/2}).
\end{align*}
\label{cor:example_bounds}
\end{corollary}
Corollary \ref{cor:example_bounds} shows that the excess worst-case risk bound of $\hat{h}_{\alpha_n,p} ^{\mathrm{prop}}$ is sharper than that of $\hat{h}_{n} ^{\mathrm{ERM}}$ if $\alpha_n \geq O(n^{-1/(2p)})$.
A typical choice of $\alpha_n$ is $ O(n^{-p/d})$ to guarantee $\bP_{\mathrm{data}} \in \mathfrak{M}_{\alpha_n, p}(\bP_n)$ with high probability \citep{shafieezadeh2019, kuhn2019}.
In such cases the proposed excess worst-case risk bound is sharper if $d>2 p^2$, but slower for the excess risk.
This shows the benefit and drawback of $\hat{h}_{\alpha_n,p} ^{\mathrm{prop}}$.
A formal statement for Corollary \ref{cor:example_bounds} and other remarks are available in the Supplementary Material.

\section{WDRO with locally perturbed data}
\label{s:wdro_perturbed}
Recently, the Mixup and its variants have led outstanding performance in many machine learning problems \citep{zhang2017, berthelot2019}.
Despite of its empirical successes, theoretical justifications of the Mixup, such as its asymptotic properties, have not been considered much in the literature.
In Section \ref{s:local_perturb}, we define locally perturbed data distributions and describe examples that include the Mixup as a special case.
Lastly, we generalize the approximation and risk consistency results presented in Section \ref{s:wdro} to the cases when data are locally perturbed in Sections \ref{s:approx_local_perturb} and \ref{s:risk_consistnecy_local_perturb}.

\subsection{Locally perturbed data distribution} 
\label{s:local_perturb}
\begin{definition}
For a dataset $\cZ_n = \{ z_1, \dots, z_n \}$ and $\beta \geq 0$, we say $\bP_n ^{\prime}$ is a $\beta$-locally perturbed data distribution if there exists a set $\{z_1 ^{\prime}, \dots, z_n ^{\prime}\}$ such that $\bP_n ^{\prime} = \frac{1}{n} \sum_{i=1} ^n \delta_{z_i ^{\prime}}$ and $z_i ^{\prime}$ can be expressed as 
\begin{align*}
z_i ^{\prime} = z_i + e_i,  
\end{align*}
for $\norm{ e_i } \leq \beta$ and $i \in [n]$. 
\end{definition}
Note that $\bP_n$ is $\beta$-locally perturbed data distribution for any $\beta \geq 0$.
The idea of locally perturbed data distribution has been widely applied in machine learning.
In the following, we provide three well known examples.

\begin{example}[Denoising autoencoder]
\citet{vincent2010} considered a set $\{z_1 ^{\prime}, \dots, z_n ^{\prime}\}$ of corrupted data defined as follows.
\begin{align*}
z_i ^{\prime} = z_i  D_i,
\end{align*}
where $D_i$ is a random diagonal matrix with diagonal elements are either one or zero.
Let $D_{(n,n)} := \max_{i \in [n] } \sup_{\norm{z} \leq 1} \norm{(I-D_i) z}$ and $\sup_{ z \in \cZ}  \norm{z} \leq C_{\cZ}$. 
Then, $\norm{(I-D_i)z_i} \leq \sup_{\norm{z}\leq 1} \norm{(I-D_i)z} C_{\cZ} \leq D_{(n,n)} C_{\cZ}$, and thus training a denoising autoencoder is equivalent to training the autoencoder using a $D_{(n,n)} C_{\cZ}$-locally perturbed data distribution.
\label{exp:denoise}
\end{example}

The next two examples deal with supervised learning settings.
For sets $\cX$ and $\cY$, suppose $\cZ = \cX \times \cY$ and $\norm{(x,y) - (\tilde{x}, \tilde{y})} = \norm{x - \tilde{x}}_{\cX} + \norm{y - \tilde{y}}_{\cY}$ for some metrics $\norm{\cdot}_{\cX}$ and $\norm{\cdot}_{\cY}$ defined on $\cX$ and $\cY$, respectively. 

\begin{example}[Mixup]
Given a dataset $\cZ_n$, we generate a Mixup dataset $\{ (x_i ^{\prime}, y_i ^{\prime})\}_{i=1} ^n$ as follows.
\begin{align*}
x_i ^{\prime} = \gamma_i x_i + (1-\gamma_i) \tilde{x}_i, \quad y_i ^{\prime} = \gamma_i y_i + (1-\gamma_i) \tilde{y}_i,
\end{align*}
for some $(\tilde{x}_i, \tilde{y}_i) \in \cZ_n$ and mixing rates $0 \leq \gamma_i \leq 1$ for all $i \in [n]$.
Let $\gamma_{(n,1)} := \min_{i \in [n]} \gamma_i$ and $\sup_{ (x, y)  \in \cZ}  \norm{(x,y)} \leq C_{\cZ}$.
Then, the Mixup dataset generates a $2 (1-\gamma_{(n,1)}) C_{\cZ}$-locally perturbed data distribution, since $\norm{ (1-\gamma_i) ((\tilde{x}_i,\tilde{y}_i) - (x_i, y_i)) } \leq 2 (1-\gamma_i) C_{\cZ} \leq 2 (1-\gamma_{(n,1)}) C_{\cZ}$ for all $i \in [n]$.
\label{exp:mixup}
\end{example}

\begin{example}[Adversarial training]
For a given dataset $\cZ_n$, \citet{goodfellow2014} proposed to minimize a loss with adversarially augmented dataset $\{(x_i ^{\prime}, y_i)\}_{i=1} ^n$.
Here, each $x_i ^{\prime} = x_i + r_i$ is newly generated data point with perturbation 
\begin{align*}
r_i := \mathrm{argmin}_{ \norm{r}_{\cX} \leq \beta_n} \log p_{\theta} ( y_i \mid x_i + r ),
\end{align*}
for some constant $\beta_n>0$ and $p_{\theta} (y \mid x)$ is a probability model parametrized by $\theta$.
From its construction, it is clear that adversarial training minimizes the risk under $\beta_n$-locally perturbed data distribution. 
Similar arguments apply to virtual adversarial training \citep{miyato2018virtual}.
\label{exp:adv}
\end{example}

\begin{remark} 
The support of a $\beta$-locally perturbed data distribution may not be a subset of $\cZ$.
Instead, it is a subset of $\cZ+\mathcal{B}(\beta) := \{ z +r \mid z \in \cZ \text{ and } \norm{r} \leq \beta \}$.
As a result, the support of a loss should be larger than $\cZ$.
\end{remark}
In the following sections, we present a rigorous analysis of WDRO with a locally augmented data distribution. 

\subsection{Approximation of the local worst-case risk}
\label{s:approx_local_perturb} 
We first show that the local worst-case risk can be approximated well by the risk under a locally perturbed data distribution when a loss is Lipschitz continuous:
\begin{proposition}
Let $(\alpha_n)$ and $(\beta_n)$ be sequences of positive numbers converging to zero and $\bP_n ^{\prime}$ be a $\beta_n$-locally perturbed data distribution. 
For a constant $M \geq \sup_{n \in \mathbb{N}} \beta_n$, assume that a loss $h:\cZ+\mathcal{B}(M) \to \bR$ is Lipschitz continuous. 
Then, for any $p \in [1, \infty)$, the following holds.
\begin{align*}
\left| R(\bP_n ^{\prime}, h) - R_{\alpha_n,p} ^{\mathrm{worst}}(\bP_n, h) \right| = O(\alpha_n \vee \beta_n).
\end{align*}
\label{prop:wdro_perturb_approx_bound}
\end{proposition}
Compared to \eqref{eq:wdro_approx_bound}, Proposition \ref{prop:wdro_perturb_approx_bound} reveals that $\beta_n$-perturbation causes an additional error $O(\beta_n)$.
This error becomes negligible when $\beta_n \leq O(\alpha_n)$.
In the following theorem, we obtain a sharper approximation result if a loss has H{\"o}lder continuous gradient (cf. Theorem \ref{thm:wdro_approx_sharp_bound}).

\begin{theorem}[Approximation to the local worst-case risk when data are perturbed]
Let $(\alpha_n)$ and $(\beta_n)$ be sequences of positive numbers converging to zero and $\bP_n ^{\prime}$ be a $\beta_n$-locally perturbed data distribution. 
Let $\cZ$ be an open and bounded subset of $\bR^d$.
For constants $C_{\mathrm{H}}, C_{\nabla} >0$, $k \in (0,1]$, and $M \geq \sup_{n \in \mathbb{N}} \beta_n$, assume that a loss $h: \mathrm{Conv}(\cZ)+\mathcal{B}(M) \to \bR$ is differentiable, its gradient $\nabla_z h(z)$ is $(C_{\mathrm{H}}, k)$-H{\"o}lder continuous, and $\bE_{\mathrm{data}}(\norm{\nabla_z h}_*) \geq C_{\nabla}$. 
Then, for $p \in (1+k, \infty)$, the following holds.
\begin{align*}
&\left| R(\bP_n ^{\prime}, h) + \alpha_n \norm{\nabla_z h}_{\bP_n ^{\prime}, p^*} -R_{\alpha_n,p} ^{\mathrm{worst}}(\bP_n, h) \right| \\
&=O_p (\alpha_n ^{1+k} \vee \beta_n).
\end{align*}
\label{thm:wdro_perturb_approx_sharp_bound}
\end{theorem}

\begin{remark}\label{remark:wdro_perturb_approx_sharp_bound}
Theorem \ref{thm:wdro_perturb_approx_sharp_bound} extends Theorem \ref{thm:wdro_approx_sharp_bound} to the cases when data are locally perturbed.
The cost of perturbation is an additional error $O(\beta_n)$, which is negligible when $\beta_n \leq O(\alpha_n ^{1+k})$. 
Thus Theorem \ref{thm:wdro_perturb_approx_sharp_bound} also suggests an appropriate size of perturbation.
\end{remark}

Based on Theorem \ref{thm:wdro_perturb_approx_sharp_bound}, for vanishing sequences $(\alpha_n)$ and $(\beta_n)$, and a $\beta_n$-locally perturbed data distribution $\bP_n ^{\prime}$, we propose to minimize the following objective function.
\begin{align}
R_{(\alpha_n, \beta_n), p} ^{\mathrm{prop}} (\bP_{n}, h) := R(\bP_n ^{\prime}, h) + \alpha_n \norm{\nabla_z h}_{\bP_n ^{\prime}, p^*},
\label{eq:prop_perturb_bound}
\end{align}
and denote its minimizer by $\hat{h}_{(\alpha_n, \beta_n), p} ^{\mathrm{prop}}$, \textit{i.e.}, $\hat{h}_{(\alpha_n, \beta_n), p} ^{\mathrm{prop}} = \mathrm{argmin}_{h \in \cH} R_{(\alpha_n, \beta_n), p} ^{\mathrm{prop}} (\bP_{n}, h)$. 

\subsection{Risk consistency of the proposed estimator}
\label{s:risk_consistnecy_local_perturb}
Now that we study risk consistency when data are locally perturbed. 
The following two theorems provide risk consistency of the minimizer $\hat{h}_{(\alpha_n, \beta_n), p} ^{\mathrm{prop}}$:

\begin{theorem}[Excess worst-case risk bound when data are perturbed]
Let $(\alpha_n)$ and $(\beta_n)$ be sequences of positive numbers converging to zero and $\bP_n ^{\prime}$ be a $\beta_n$-locally perturbed data distribution. 
Let $\cZ$ be an open and bounded subset of $\bR^d$.
For constants $C_{\mathrm{H}}, C_{\nabla}, L >0$, $k \in (0,1]$, and $M \geq \sup_{n \in \mathbb{N}} \beta_n$, assume that $\cH$ is a uniformly bounded set of differentiable functions $h: \mathrm{Conv}(\cZ)+\mathcal{B}(M) \to \mathbb{R}$ such that its gradient $\nabla_z h$ is $(C_{\mathrm{H}}, k)$-H{\"o}lder continuous, $\bE_{\mathrm{data}}(\norm{\nabla_z h}_*) \geq C_{\nabla}$, and $\mLip(h) \leq L$.
Then, for $p \in (1+k, \infty)$, the following holds.
\begin{align*}
&\cE_{\alpha_n,p} ^{\mathrm{worst}}( \hat{h}_{(\alpha_n, \beta_n), p} ^{\mathrm{prop}}) \\
&= O_p \left( \frac{\mathfrak{C}(\cH) \vee \alpha_n ^{1-p}}{\sqrt{n}} \vee \log(n) ( \alpha_n ^{1+k} \vee \beta_n) \right).
\end{align*}
\label{thm:wdro_perturb_excess_worst_risk}
\end{theorem}
\begin{theorem}[Excess risk bound when data are perturbed]
Under the same assumptions as Theorem \ref{thm:wdro_perturb_excess_worst_risk}, the following holds.
\begin{align*}
&\cE( \hat{h}_{(\alpha_n, \beta_n), p} ^{\mathrm{prop}}) \\
&= O_p( \mathfrak{R}_n (\mathcal{H}) \vee n^{-1/2} \vee \alpha_n \vee \log(n)( \alpha_n ^{1+k} \vee \beta_n) ).
\end{align*} 
\label{thm:wdro_perturb_excess_risk}
\end{theorem}
Similar to Remark \ref{remark:wdro_perturb_approx_sharp_bound}, the errors due to the local perturbation both in the order of $\log(n) \beta_n$ are negligible when $\beta_n \leq O(\alpha_n ^{1+k})$.
In such settings, Theorems \ref{thm:wdro_perturb_excess_worst_risk} and \ref{thm:wdro_perturb_excess_risk} yield the same bound as Theorems \ref{thm:wdro_excess_worst_risk} and \ref{thm:wdro_excess_risk}, respectively.

\begin{table*}[t]
\caption[Accuracy comparison of the four methods using the clean and noisy test datasets with various training sample sizes]{Accuracy comparison of the four methods using the clean and noisy test datasets with various training sample sizes. Average and standard deviation are denoted by \lq{}average$\pm$standard deviation\rq{}. All the results are based on five independent trials. Boldface numbers denote the best and equivalent methods  with respect to a t-test with a significance level of 5\%.}
\label{t:compare}
\vskip 0.15in
\begin{center}
\begin{small}
\begin{sc}
\resizebox{\textwidth}{!}{
\begin{tabular}{lcccc|cccc}
\toprule
Sample &\multicolumn{4}{c}{Clean} & \multicolumn{4}{c}{ 1\% salt and pepper noise} \\
\cmidrule{2-5} \cmidrule{6-9}
size & ERM & WDRO & MIXUP & WDRO+MIX & ERM & WDRO & MIXUP & WDRO+MIX \\
\midrule
CIFAR-10\\
2500 &$77.3\pm0.8$ & $77.1\pm0.7$ &$\mathbf{81.4\pm0.5}$ & $\mathbf{80.8\pm0.7}$ & $69.8\pm1.8$ & $71.9\pm0.9$ & $72.7\pm1.6$ & $\mathbf{74.8\pm0.9}$\tabularnewline
5000 &$83.3\pm0.4$ & $83.0\pm0.3$ & $\mathbf{86.7\pm0.2}$ & $\mathbf{85.6\pm0.3}$ & $75.2\pm1.4$ &  $77.4\pm0.5$ & $76.4\pm1.7$ & $\mathbf{79.6\pm0.9}$ \tabularnewline
25000 &$92.2\pm0.2$ & $91.4\pm0.1$ & $\mathbf{93.3\pm0.1}$  &$92.4\pm0.1$ & $83.3\pm0.8$ & $\mathbf{85.8\pm0.5}$ & $82.1\pm1.7$ & $\mathbf{86.2\pm0.3}$\tabularnewline
50000 &$94.1\pm0.1$ & $93.1\pm0.1$ &$\mathbf{94.8\pm0.2}$ &$93.5\pm0.2$ & $84.1\pm1.0$ & $\mathbf{87.4\pm0.5}$ & $82.5\pm1.3$ & $\mathbf{87.3\pm0.5}$ \tabularnewline 
\midrule
CIFAR-100\\
2500 &$33.8\pm1.0$ & $34.6\pm1.7$ &$\mathbf{38.9\pm0.6}$ &$\mathbf{39.4\pm0.2}$ & $29.2\pm0.2$ & $30.4\pm1.2$ & $33.2\pm1.1$ & $\mathbf{35.0\pm0.5}$\tabularnewline
5000 &$45.2\pm0.9$ & $43.7\pm0.7$ & $\mathbf{49.9\pm0.2}$ &$\mathbf{49.5\pm0.4}$ & $37.0\pm0.8$ & $38.1\pm1.1$ & $39.4\pm1.3$ & $\mathbf{42.3\pm0.7}$ \tabularnewline
25000 &$67.8\pm0.2$ & $66.6\pm0.3$ & $\mathbf{69.3\pm0.3}$ &$68.2\pm0.3$ & $51.0\pm1.9$ & $\mathbf{56.5\pm0.8}$ & $49.6\pm1.0$ & $\mathbf{55.8\pm0.4}$ \tabularnewline
50000 &$74.4\pm0.2$ & $73.5\pm0.3$ & $\mathbf{75.2\pm0.2}$ & $73.8\pm0.3$ & $51.9\pm1.3$ & $\mathbf{62.1\pm0.5}$ & $50.0\pm3.0$ & $60.6\pm0.7$ \tabularnewline
\bottomrule
\end{tabular}
}
\end{sc}
\end{small}
\end{center}
\vskip -0.1in
\end{table*}

\begin{remark} 
By setting $\beta_n = 2 (1-\gamma_{(n,1)}) C_{\cZ}$, all the theorems presented in Section \ref{s:wdro_perturbed} apply to the Mixup (see Example \ref{exp:mixup}).
To make sure $\lim_{n \to \infty} \beta_n = 0$, we need $\lim_{n \to \infty} \gamma_{(n,1)} = 1$ and it can be satisfied as long as we do not perturb the original data too much as the sample size increases.
Similar arguments are applicable to Examples \ref{exp:denoise} and \ref{exp:adv}.
\end{remark}

\section{Numerical experiments}
\label{s:numerical_experiments}

In this section, we conduct numerical experiments to demonstrate robustness of the proposed method using image classification datasets.

\textbf{Methods}
We consider the four methods: (i) the empirical risk minimization, denoted by ERM, (ii) the proposed method based on \eqref{eq:prop_bound}, denoted by WDRO, (iii) the empirical risk minimization with the Mixup, denoted by MIXUP, and (iv) the proposed method with the Mixup based on \eqref{eq:prop_perturb_bound}, denoted by WDRO+MIX.

\textbf{Datasets}
We use the two image classification datasets: CIFAR-10 and CIFAR-100 \citep{Krizhevsky09}. 
For the training, we randomly select 2500, 5000, 25000, or 50000 images from the original datasets, keeping the number of images per class equal. 
For the testing, we use the original test datasets.

Further implementation details are available in the Supplementary Material and Tensorflow \citep{abadi2016}-based scripts are available at \url{https://github.com/ykwon0407/wdro_local_perturbation}.

\subsection{Accuracy comparison}\label{subsec:robustness_of_GP}
To evaluate robustness of the methods, we compute accuracy on both clean and contaminated datasets.
For the latter, we apply the salt and pepper noise to the clean images \citep{hwang1995}.
Figure \ref{fig:cifar_sample} displays an example of the clean and contaminated images used in our experiments.

\textbf{Experiment 1 }
In this experiment, we compare the accuracy of the four methods using the clean and contaminated datasets. 
For the contaminated datasets, we apply the salt and pepper noise to 1\% of pixels.
The training sample sizes vary as $2500$, $5000$, $25000$, and $50000$.
We repeatedly select samples and train models five times\footnote{Note that when the sample size is $50000$, the training set is fixed but trained models can vary due to the randomness of algorithms.}.

Table \ref{t:compare} compares accuracy of the four methods.
For the clean datasets, WDRO+MIX performs comparably with MIXUP and achieves significantly higher accuracy than ERM and WDRO when the sample sizes are 2500 and 5000.
When the sample sizes are 25000 and 50000, either WDRO or WDRO+MIX shows lower accuracy than MIXUP and ERM.
For the contaminated datasets, either WDRO or WDRO+MIX achieves significantly higher accuracy than ERM and MIXUP in all settings.
This shows that the proposed method is robust to contamination of data.

\begin{figure}[t]
\vskip 0.2in
\begin{center}
\centerline{
\includegraphics[width=0.24\columnwidth]{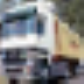}
\includegraphics[width=0.24\columnwidth]{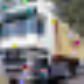}
\includegraphics[width=0.24\columnwidth]{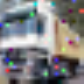}
\includegraphics[width=0.24\columnwidth]{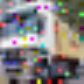}
}
\caption{An example of clean and contaminated images. We apply the salt and pepper noise to each pixel. The probabilities of noisy pixels are $0\%$, $1 \%$, $2\%$, and $4\%$ from left to right, respectively.}
\label{fig:cifar_sample}
\end{center}
\vskip -0.2in
\end{figure}

\textbf{Experiment 2 } 
In this experiment, we compare the reduction of the accuracy from using the clean datasets to the contaminated datasets. 
For the noise intensity, the probabilities of noisy pixels are set to $1\%$, $2\%$, and $4\%$.
We repeatedly train models five times using the original 50000 images.

Table \ref{t:noise_gap} shows accuracy reduction of the four methods.
The WDRO, with or without the Mixup, achieves a significantly lower reduction than ERM and MIXUP in every noise level and dataset.
For example, on CIFAR-10, the accuracy reduction in WDRO+MIX is $12.7\%$ on average, compared to $24.3\%$ in MIXUP when the probability of noisy pixels is $2\%$.
With the same noise level, on CIFAR-100, the accuracy reduction in WDRO+MIX is $29.7\%$ on average, compared to $45.9\%$ in MIXUP.
This result shows that the proposed method produces robust models against various noise levels.

\begin{table}[t]
\caption{The comparison of the accuracy reduction on various salt and pepper noise intensities. Other details are given in Table \ref{t:compare}.}
\label{t:noise_gap}
\vskip 0.15in
\begin{center}
\begin{small}
\begin{sc}
\resizebox{\columnwidth}{!}{
\begin{tabular}{lccccccccc}
\toprule
Probability of & \multirow{2}{*}{ERM} & \multirow{2}{*}{WDRO} & \multirow{2}{*}{MIXUP} & \multirow{2}{*}{WDRO+MIX} \\
noisy pixels &  &  & &  \tabularnewline 
\midrule
CIFAR-10 & & & &  \tabularnewline
1\% & $10.1\pm 0.9$ & $\mathbf{5.7\pm 0.4}$ & $12.4\pm1.2$ & $\mathbf{6.2\pm0.4}$ \tabularnewline
2\% & $21.1\pm1.9$ & $\mathbf{13.2\pm 0.5}$ & $24.3\pm1.4$ & $\mathbf{12.7\pm0.8}$ \tabularnewline
4\% & $39.7\pm2.9$ & $\mathbf{32.9\pm 2.5}$ & $43.5\pm1.8$ & $\mathbf{30.9\pm2.0}$ \tabularnewline
\midrule
CIFAR-100 & & & & \tabularnewline
1\% & $22.5\pm1.3$ & $\mathbf{11.4\pm0.4}$ & $25.2\pm2.5$ & $13.2\pm0.7$ \tabularnewline
2\% & $42.8\pm2.3$ & $\mathbf{26.5\pm1.0}$ & $45.9\pm3.4$ & $29.7\pm0.7$ \tabularnewline
4\% & $61.7\pm1.4$ & $\mathbf{50.0\pm0.9}$ & $63.9\pm2.0$ & $53.5\pm0.9$ \tabularnewline
\bottomrule
\end{tabular}
}
\end{sc}
\end{small}
\end{center}
\vskip -0.1in
\end{table}

\begin{figure}[t]
\vskip 0.2in
\begin{center}
\centerline{\includegraphics[width=0.9\columnwidth, height=2.6in]{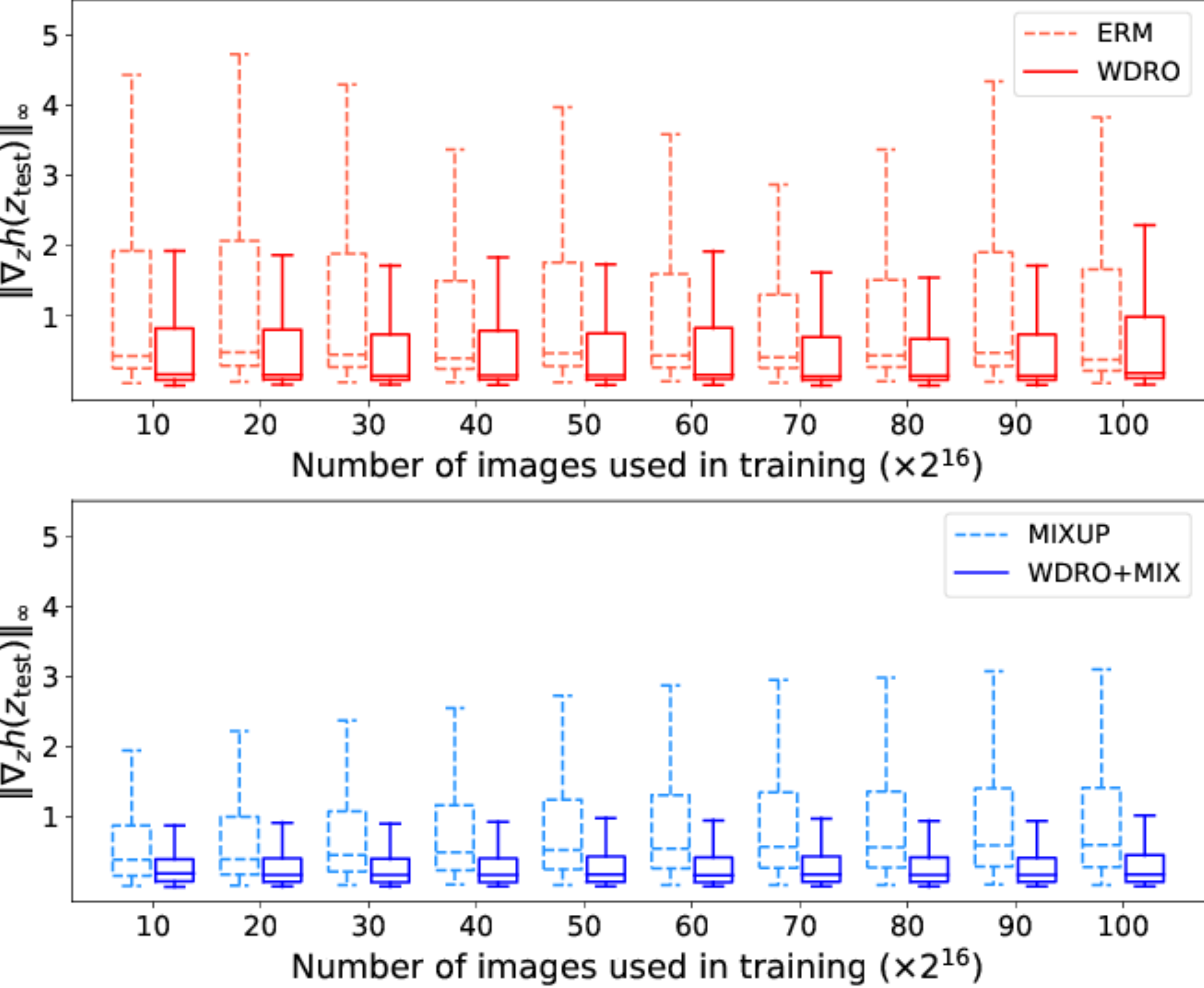}}
\caption[The distribution of the $\ell_{\infty}$-norm of the gradients]{The box plots of the $\ell_{\infty}$-norm of the gradients when the number of images used in training increases from $10 \times 2^{16}$ to $100 \times 2^{16}$. We use the original CIFAR-10 test images. The box plots on the top represent the gradient distribution of (dashed) ERM and (solid) WDRO, respectively, and the box plots on the bottom represent that of (dashed) MIXUP and (solid) WDRO+MIX, respectively.
}
\label{fig:acc_and_grad}
\end{center}
\vskip -0.2in
\end{figure}

\begin{figure}[t]
\vskip 0.2in
\begin{center}
\centerline{\includegraphics[width=0.95\columnwidth, height=2.8in]{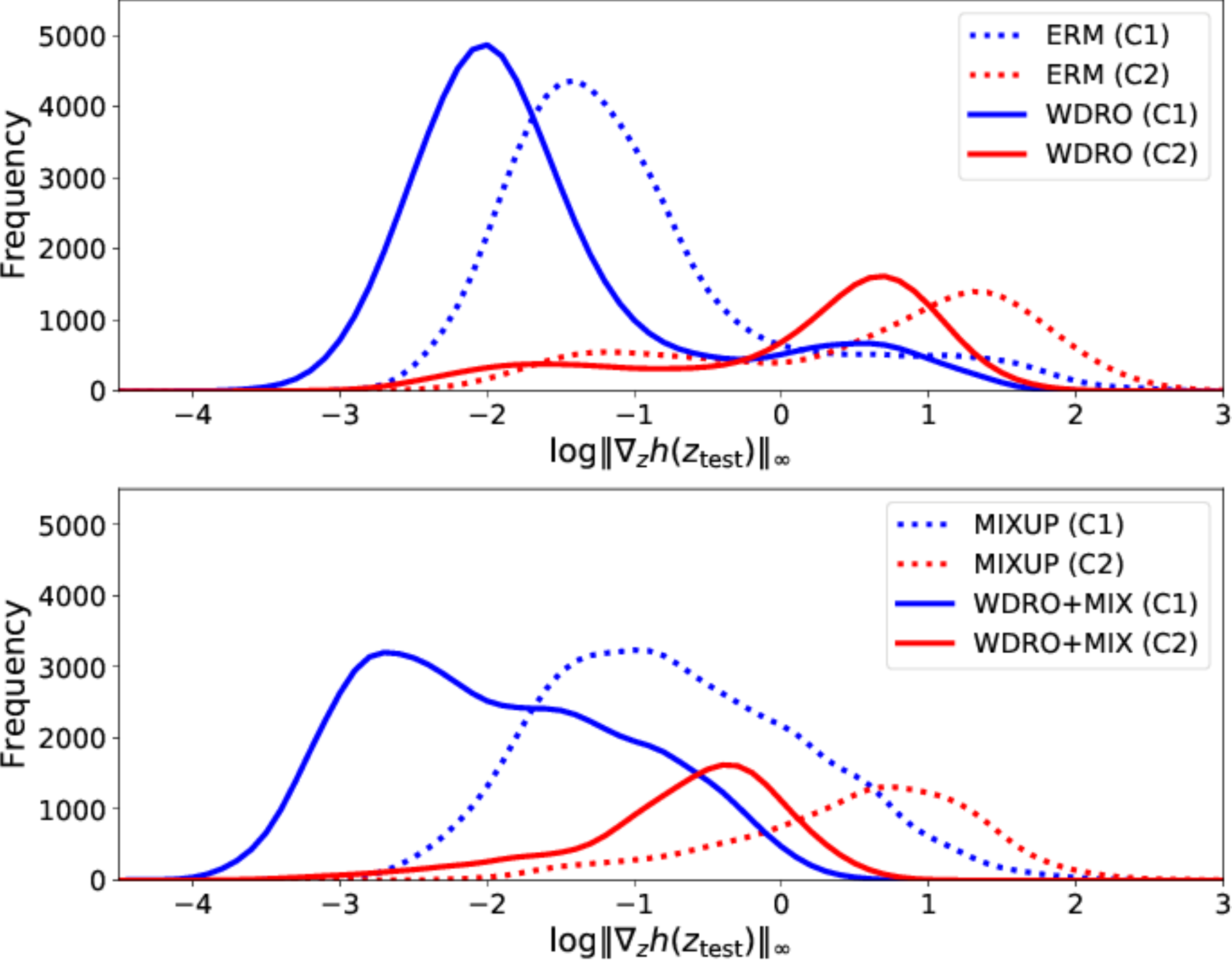}}
\caption{The smoothed histograms of gradients characterizing the two categories: (C1) the images that are correctly classified on both clean and contaminated state and (C2) the other images, respectively. The top panel shows the histograms of ERM (dotted) vs. WDRO (solid) and the bottom panel shows those of MIXUP (dotted) vs. WDRO+MIX (solid), respectively.}
\label{fig:grad_and_robustness}
\end{center}
\vskip -0.2in
\end{figure}

\subsection{Analysis of the gradient}\label{subsec:gradient_penalization}
In this subsection, we demonstrate robustness of our method by analyzing the distribution of the gradients of the loss.
We train models with randomly selected 5000 CIFAR-10 images and evaluate the gradients with the CIFAR-10 test images.
We consider the $\ell_{\infty}$-norm of the gradients $\norm{\nabla_z h(z_{\rm test})}_{\infty}$ for each test image $z_{\rm test}$.

\textbf{Experiment 3 }
In this experiment, we compare the gradients of ERM vs. WDRO and MIXUP vs. WDRO+MIX as the number of images used in training increases, respectively.
Figure \ref{fig:acc_and_grad} shows the box plots of the $\ell_{\infty}$-norm of the gradients.
Over the entire training phases, the first and third quartiles of the gradients of WDRO (\textit{resp.} WDRO+MIX) are smaller than those of the gradients of ERM (\textit{resp.} MIXUP).
This result empirically validates that the gradient penalties in \eqref{eq:prop_bound} and \eqref{eq:prop_perturb_bound} lead to small gradients and robustness of WDRO and WDRO+MIX.

\textbf{Experiment 4 } 
We visualize smoothed histograms of the gradients for the four methods.
We divide the test datasets into the following two categories: (C1) images that are correctly classified on both clean and contaminated state and (C2) images that are incorrectly classified on either clean or contaminated state.
In this experiment, the level of noise for the contaminated dataset is $1\%$. 

Figure \ref{fig:grad_and_robustness} shows the smoothed histograms of the gradients for (C1) and (C2).
The gradients for (C1) depicted in blue are smaller than those for (C2) depicted in red.
In both categories (C1) and (C2), the WDRO (\textit{resp.} WDRO+MIX), depicted in the solid line, has smaller gradients than ERM (\textit{resp.} MIXUP) depicted in the dotted line. 
Thus the proposed method tends to reduce the sizes of gradients in both categories (C1) and (C2), which leads to the robustness of WDRO and WDRO+MIX.

\section{Conclusion}
In this work, we have developed a principled and tractable statistical inference method for WDRO.
In addition, we formally present a locally perturbed data distribution (e.g., Mixup) and develop WDRO inference when data are locally perturbed.
Numerical experiments demonstrate robustness of the proposed method.

\section*{Acknowledgments}
Yongchan Kwon was supported by the National Research Foundation of Korea (NRF) grant funded by the Korea government (MSIT, No.2017R1A2B4008956).
Wonyoung Kim and Myunghee Cho Paik were supported by the NRF grant (MSIT, No.2020R1A2C1A0101195011) and Joong-Ho Won was supported by the NRF grant (MSIT, No.2019R1A2C1007126).
Wonyoung Kim was also supported by Hyundai Chung Mong-koo foundation.

\bibliography{ref}
\bibliographystyle{icml2020}

\appendix

\onecolumn
\icmltitle{Appendix: Principled Learning Method for Wasserstein Distributionally Robust Optimization with Local Perturbations}

\setcounter{equation}{8}
\setcounter{remark}{5}
\setcounter{theorem}{6}
\setcounter{proposition}{1}
\setcounter{corollary}{1}
\setcounter{table}{3}

\section{Proofs}
When $M = 0$ and $\beta_n = 0$ for all $n$, a $\beta_n$-locally perturbed data distribution is the empirical data distribution, \textit{i.e.},  $\bP_n ^{\prime}=\bP_n$.  
Therefore, Theorem \ref{thm:wdro_approx_sharp_bound} is a special case of Theorem \ref{thm:wdro_perturb_approx_sharp_bound}.
Also, in such cases, $R_{\alpha_n, p} ^{\mathrm{prop}} (\bP_{n}, h) = R_{(\alpha_n, \beta_n), p} ^{\mathrm{prop}} (\bP_{n}, h)$ and $\hat{h}_{\alpha_n,p} ^{\mathrm{prop}} = \hat{h}_{(\alpha_n, \beta_n),p} ^{\mathrm{prop}}$, and Theorems \ref{thm:wdro_excess_worst_risk} and thus \ref{thm:wdro_excess_risk} are a special case of Theorems \ref{thm:wdro_perturb_excess_worst_risk} and \ref{thm:wdro_perturb_excess_risk}, respectively. 
In this respect, we omit proofs for Theorems \ref{thm:wdro_approx_sharp_bound}, \ref{thm:wdro_excess_worst_risk}, and \ref{thm:wdro_excess_risk}.

\subsection{Proof of Proposition \ref{prop:wdro_perturb_approx_bound}}
\label{s:wdro_perturb_approx_bound_proof}
\begin{proof}[Proof of Proposition \ref{prop:wdro_perturb_approx_bound}]
Since $\bP_n \in \fM_{\alpha_n,p} (\bP_n)$, we have 
\begin{align*}
R(\bP_n, h) \leq R_{\alpha_n,p} ^{\mathrm{worst}}(\bP_n, h).
\end{align*}
Let $\bQ^*$ be such that 
$R(\bQ^*, h) =  \sup_{\bQ \in \fM_{\alpha_n,p} (\bP_n)} R(\bQ, h) = R_{\alpha_n,p} ^{\mathrm{worst}}(\bP_n, h)$.
Since $h$ is Lipschitz continuous, the Kantorovich-Rubinstein duality \citep[Remark 6.5]{villani2008} gives
\begin{align*}
R(\bQ^*, h) - R(\bP_n, h) &\leq \mLip(h) \cW_1 (\bQ^*, \bP_n) \\
&\leq \mLip(h) \cW_p (\bQ^*, \bP_n) \\
&\leq \mLip(h) \alpha_n.
\end{align*}
Here, the second inequality is due to $\cW_1 (\bQ^*, \bP_n) \leq \cW_p (\bQ^*, \bP_n)$ for $p \in [1, \infty)$ \citep[Remark 6.6]{villani2008}.
Thus, 
\begin{align}
\left| R(\bP_n, h) - R_{\alpha_n,p} ^{\mathrm{worst}}(\bP_n, h) \right| \leq \mLip(h) \alpha_n.
\label{ineq:diff_worst_emp}
\end{align}

Write $\bP_n ^{\prime} = \frac{1}{n} \sum_{i=1} ^n \delta_{z_i ^{\prime}}$ for some $\{z_1', \dotsc, z_n'\}$ such that $\norm{z_i ^{\prime} - z_i} \leq \beta_n$ for all $i \in [n]$.
Then, we have $z_i ^{\prime} \in \cZ + \mathcal{B}(M)$ and $h(z_i ^{\prime})$\rq{}s are well defined.
By the Lipschitz continuity of $h$ and the definition of $\bP_n ^{\prime}$, we have 
\begin{align*}
\left| R(\bP_n, h) - R(\bP_n ^{\prime}, h) \right| &= \left| \frac{1}{n} \sum_{i=1} ^n (h(z_i)-h(z_i ^{\prime})) \right| \\
&\leq \frac{1}{n} \sum_{i=1} ^n \mLip(h) \norm{z_i ^{\prime} - z_i} \\
&\leq \mLip(h) \beta_n.
\end{align*}
Therefore, we have
\begin{align*}
\left| R(\bP_n ^{\prime}, h) - R_{\alpha_n,p} ^{\mathrm{worst}}(\bP_n, h)  \right| \leq (\alpha_n + \beta_n  )\mLip(h).
\end{align*}
This concludes the proof.
\end{proof}

\subsection{Proof of Theorem \ref{thm:wdro_perturb_approx_sharp_bound}}
\label{s:wdro_perturb_approx_sharp_bound_proof}

\begin{proof}[Proof of Theorem \ref{thm:wdro_perturb_approx_sharp_bound}]
Write $\bP_n ^{\prime} = \frac{1}{n} \sum_{i=1} ^n \delta_{z_i ^{\prime}}$ for some $\{z_1', \dotsc, z_n'\}$ such that
$\norm{z_i ^{\prime} - z_i} \leq \beta_n$
for all $i \in [n]$.
Then, we have $z_i ^{\prime} \in \mathrm{Conv}(\cZ) + \mathcal{B}(M)$ and $h(z_i ^{\prime})$\rq{}s are well defined. 

[Step 1] In this step we first establish an upper bound for the local worst-case risk $R_{\alpha_n,p} ^{\mathrm{worst}}(\bP_n, h)$. 
Since $h$ is well defined and differentiable on $\mathrm{Conv}(\cZ) + \mathcal{B}(M)$, we can apply the mean value theorem.
Due to the $(C_{\mathrm{H}},k)$-H{\"o}lder continuity of $\nabla_z h$, for any $i \in [n]$ and $\tilde{z}_i \in \cZ$, we have
\begin{align*}
h(\tilde{z}_i) &= h(z_i ^{\prime}) + \langle \nabla_z h(c_i), \tilde{z}_i -z_i ^{\prime} \rangle \\
&= h(z_i ^{\prime}) + \langle \nabla_z h(z_i ^{\prime}),  \tilde{z}_i -z_i ^{\prime} \rangle + \langle \nabla_z h(c_i) - \nabla_z h(z_i ^{\prime}), \tilde{z}_i -z_i ^{\prime} \rangle \\
&\leq h(z_i ^{\prime}) + \norm{\nabla_z h(z_i ^{\prime})}_* \norm{ \tilde{z}_i -z_i ^{\prime} } + C_{\mathrm{H}} \norm{\tilde{z}_i - z_i ^{\prime}} ^{1+k},
\end{align*}
where $c_i = \tau_i z_i ^{\prime} + (1-\tau_i) \tilde{z}_i$ for some $\tau_i \in [0,1]$.
By the triangle inequality and Jensen's inequality, $(a+b)^{1+k} \leq 2^k (a^{1+k}+ b^{1+k})$ for any $a, b \geq 0$, we have
\begin{align*}
&h(z_i ^{\prime}) + \norm{\nabla_z h(z_i ^{\prime})}_* \norm{ \tilde{z}_i -z_i ^{\prime}} + C_{\mathrm{H}} \norm{\tilde{z}_i - z_i ^{\prime}} ^{1+k} \\
&\leq h(z_i ^{\prime}) + \norm{\nabla_z h(z_i ^{\prime})}_* (\beta_n + \norm{ \tilde{z}_i -z_i } )+ C_{\mathrm{H}} 2^k ( \norm{\tilde{z}_i - z_i } ^{1+k} + \beta_n ^{1+k}) \\
&= \beta_n\left( \norm{\nabla_z h(z_i ^{\prime})}_* + C_{\mathrm{H}} 2^k \beta_n ^k \right) + h(z_i ^{\prime}) + \norm{\nabla_z h(z_i ^{\prime})}_* \norm{ \tilde{z}_i -z_i } + C_{\mathrm{H}} 2^k \norm{\tilde{z}_i - z_i } ^{1+k}.
\end{align*}
To this end, we set $C_{\mathrm{H},k} := C_{\mathrm{H}} 2^k$ and  $t_i := \norm{ \tilde{z}_i -z_i }$.
By \citet[Lemma 2]{gao2017}, for any $\eta > 0$ and $\lambda \geq 0$, we have
\begin{align*}
&\norm{\nabla_z h(z_i ^{\prime})}_* t_i + C_{\mathrm{H},k} t_i ^{1+k} - \lambda t_i ^{p} \\
&\leq \left( \norm{\nabla_z h(z_i ^{\prime})}_* + \frac{p-k-1}{p-1} C_{\mathrm{H},k} \eta \right) t_i - \left(\lambda - \frac{k}{p-1}C_{\mathrm{H},k} \eta^{-\frac{p-k-1}{k}} \right) t_i ^{p}.
\end{align*}
By substituting $\eta$ with $\alpha_n ^k$, 
\begin{align}
&\norm{\nabla_z h(z_i ^{\prime})}_* t_i + C_{\mathrm{H},k} t_i ^{1+k} - \lambda t_i ^{p} \notag \\
&\leq \left(\norm{\nabla_z h(z_i ^{\prime})}_* + \frac{p-k-1}{p-1} C_{\mathrm{H}, k} \alpha_n ^k \right) t_i - \left(\lambda - \frac{k}{p-1}C_{\mathrm{H},k} \alpha_n^{-(p-k-1)}\right) t_i ^{p} \notag \\
&=:  h_{\alpha_n} \left(z_i ^{\prime}) t_i - (\lambda - C_{\alpha_n} \right) t_i ^p. \label{ineq:upper_part_1}
\end{align}
Since $\cZ$ is bounded, there exists a constant $D_{\cZ}$ such that $\sup_{z, \tilde{z} \in \cZ} \norm{z - \tilde{z}} \leq D_{\cZ}$. 
Then,
\begin{align*}
\sup_{0\leq t \leq D_{\cZ}} \{ h_{\alpha_n} (z_i ^{\prime}) t - (\lambda - C_{\alpha_n}) t^p \} = \begin{cases}
h_{\alpha_n} (z_i ^{\prime}) D_{\cZ} -(\lambda - C_{\alpha_n}) D_{\cZ}^p &\text{if } 0 \leq \lambda \leq C_{\alpha_n},\\
h_{\alpha_n} (z_i ^{\prime}) t_* (\lambda) -(\lambda - C_{\alpha_n}) t_* ^p (\lambda) &\text{if }  C_{\alpha_n} < \lambda ,
\end{cases}
\end{align*}
where $t_* (\lambda) = \min\left\{ \left(\frac{h_{\alpha_n} (z_i ^{\prime})}{(\lambda-C_{\alpha_n}) p} \right)^{1/(p-1)}, D_{\cZ}\right\}$.
Here,
\begin{align*}
\left(\frac{ h_{\alpha_n} (z_i ^{\prime})}{(\lambda-C_{\alpha_n}) p} \right)^{1/(p-1)} < D_{\cZ} \Leftrightarrow C_{\alpha_n} + \frac{h_{\alpha_n} (z_i ^{\prime})}{p D_{\cZ}^{p-1}} < \lambda.
\end{align*}
Thus, 
\begin{align*}
\sup_{0\leq t \leq D_{\cZ}} \{ h_{\alpha_n} (z_i ^{\prime}) t - (\lambda - C_{\alpha_n}) t^p \} = \begin{cases}
h_{\alpha_n} (z_i ^{\prime}) D_{\cZ} -(\lambda - C_{\alpha_n}) D_{\cZ}^p, &\text{if } 0 \leq  \lambda \leq C_{\alpha_n} + \frac{h_{\alpha_n} (z_i ^{\prime})}{p D_{\cZ}^{p-1}}, \\
p^{-p^*} (p-1) (\lambda - C_{\alpha_n})^{-\frac{1}{p-1}} \norm{ h_{\alpha_n}(z_i ^{\prime}) }_* ^{p^*}, & \text{if } C_{\alpha_n} + \frac{h_{\alpha_n} (z_i ^{\prime})}{p D_{\cZ}^{p-1}} < \lambda.
\end{cases}
\end{align*}
Note that $\norm{ h_{\alpha_n}(z_i ^{\prime}) }_* = h_{\alpha_n}(z_i ^{\prime})$.
Let $ \lambda_* :=  C_{\alpha_n} + \frac{ \max_{i \in [n]} \{ h_{\alpha_n} (z_i ^{\prime})\} }{p D_{\cZ}^{p-1}}$.
Using the triangle inequality and the H{\"o}lder continuity of $\nabla_z h$, for any $z\in \cZ$ and some point $z_0 \in \cZ$, we have
\begin{align*}
\norm{\nabla_z h (z)}_* &\leq \norm{\nabla_z h (z_0)}_* + \norm{\nabla_z h (z)- \nabla_ z h (z_0)}_* \\
&\leq  \norm{\nabla_z h (z_0)}_* + C_{\mathrm{H}} \norm{z - z_0}^{k} \\
&\leq  \norm{\nabla_z h (z_0)}_* + C_{\mathrm{H}} D_{\cZ}^{k}.
\end{align*}
This implies $\norm{\nabla_z h (z)}_{*} $ is bounded for all $z \in \mathrm{Conv}(\cZ) + \mathcal{B}(M)$. We denote the upper bound by $L_{\nabla}$, \textit{i.e.}, $\norm{\nabla_z h (z)}_{*} \leq L_{\nabla} <\infty$ for all $z \in \mathrm{Conv}(\cZ) + \mathcal{B}(M)$.
Then, we have
\begin{align}
\frac{ \max_{i \in [n]} \{ h_{\alpha_n} (z_i ^{\prime})\} }{p D_{\cZ}^{p-1}} \leq  \frac{ L_{\nabla} + \frac{p-k-1}{p-1} C_{\mathrm{H}, k} \alpha_n ^k}{p D_{\cZ}^{p-1}} < \infty.
\label{ineq:approx_upper_part_1}
\end{align}
At the same time, by the definition of $\norm{ h_{\alpha_n} }_{\bP_n ^{\prime}, 1}$, we have
\begin{align}
\frac{ 0 + \frac{p-k-1}{p-1} C_{\mathrm{H}, k} \alpha_n ^k}{p  \alpha_n ^{p-1}} \leq \frac{\norm{ h_{\alpha_n} }_{\bP_n ^{\prime}, 1}  }{ p \alpha_n ^{p-1}}, 
\label{ineq:approx_upper_part_2}
\end{align}
and the left-hand side diverges to infinity as $n$ increases due to $p >1+k$.
Since $\norm{ h_{\alpha_n} }_{\bP_n ^{\prime}, 1} \leq \norm{ h_{\alpha_n} }_{\bP_n ^{\prime}, p^*}$ and by the inequalities \eqref{ineq:approx_upper_part_1} and \eqref{ineq:approx_upper_part_2} give for a large enough $n$,
\begin{align*}
\lambda_* < C_{\alpha_n} + \frac{ \norm{ h_{\alpha_n} }_{\bP_n ^{\prime}, p^*}  }{p \alpha_n ^{p-1}}.
\end{align*}
Therefore, for a large enough $n$,
\begin{align}
\inf_{ \lambda_*< \lambda} \left\{ \lambda \alpha_n ^{p} + \frac{1}{n} \sum_{i=1} ^n \sup_{0\leq t \leq D_{\cZ}} \{ h_{\alpha_n} (z_i ^{\prime}) t - (\lambda-C_{\alpha_n}) t^p \} \right\} &=  C_{\alpha_n} \alpha_n ^{p}  + \alpha_n  \norm{ h_{\alpha_n} }_{\bP_n ^{\prime}, p^*} \notag \\
&\leq C_{\alpha_n} \alpha_n ^{p}  + \alpha_n \left\{ \norm{\nabla_z h}_{\bP_n ^{\prime}, p^*} + \frac{p-k-1}{p-1} C_{\mathrm{H},k} \alpha_n ^k \right\} \notag \\
&=\alpha_n \norm{\nabla_z h}_{\bP_n ^{\prime}, p^*} + C_{\mathrm{H},k} \alpha_n ^{1+k}.  \label{ineq:upper_part_2}
\end{align}
The inequality is due to the Minkowski inequality.
By arranging all the results, for a large enough $n$, we have
\begin{align*}
&R_{\alpha_n,p} ^{\mathrm{worst}}(\bP_n, h) - R(\bP_n ^{\prime}, h)  \\
&\stackrel{\eqref{eq:dro_duality}}{=} \min_{\lambda \geq 0} \Big\{ \lambda \alpha_n^p + \frac{1}{n} \sum_{i=1} ^n \sup_{\tilde{z} \in \cZ } \left\{ h(\tilde{z}) - h(z_i ^{\prime})- \lambda \norm{\tilde{z} - z_i}^p \right\} \Big\}  \\
&\leq \beta_n( \norm{\nabla_z h}_{\bP_n ^{\prime},1} + C_{\mathrm{H},k} \beta_n ^k ) + \min_{\lambda \geq 0} \Big\{ \lambda \alpha_n^p + \frac{1}{n} \sum_{i=1} ^n \sup_{\tilde{z} \in \cZ } \left\{  \norm{\nabla_z h(z_i ^{\prime})}_* \norm{ \tilde{z}_i -z_i } + C_{\mathrm{H}, k} \norm{\tilde{z}_i - z_i } ^{1+k} - \lambda \norm{\tilde{z} - z_i}^p \right\} \Big\}  \\
&\leq \beta_n( \norm{\nabla_z h}_{\bP_n ^{\prime},1} + C_{\mathrm{H},k} \beta_n ^k ) + \min_{\lambda \geq 0} \Big\{ \lambda \alpha_n^p + \frac{1}{n} \sum_{i=1} ^n \sup_{ 0 \leq t \leq D_{\cZ} } \left\{  \norm{\nabla_z h(z_i ^{\prime})}_* t + C_{\mathrm{H}, k} t ^{1+k} - \lambda t^p \right\} \Big\}  \\
&\stackrel{\eqref{ineq:upper_part_1}}{\leq} \beta_n( \norm{\nabla_z h}_{\bP_n ^{\prime},1} + C_{\mathrm{H},k} \beta_n ^k ) + \min_{\lambda \geq \lambda_* } \Big\{ \lambda \alpha_n^p + \frac{1}{n} \sum_{i=1} ^n \sup_{ 0 \leq t_i \leq D_{\cZ} } \left\{  h_{\alpha_n} \left(z_i ^{\prime}) t_i - (\lambda - C_{\alpha_n} \right) t_i ^p \right\} \Big\}  \\
&\stackrel{\eqref{ineq:upper_part_2}}{\leq} \beta_n( \norm{\nabla_z h}_{\bP_n ^{\prime},1} + C_{\mathrm{H},k} \beta_n ^k ) + \alpha_n \norm{\nabla_z h}_{\bP_n ^{\prime}, p^*} + C_{\mathrm{H},k} \alpha_n ^{1+k}  \\
&= O(\beta_n + \alpha_n ^{1+k}) + \alpha_n \norm{\nabla_z h}_{\bP_n ^{\prime}, p^*}.
\end{align*}
Thus, we have
\begin{align}
R_{\alpha_n,p} ^{\mathrm{worst}}(\bP_n, h) - R(\bP_n ^{\prime}, h) - \alpha_n \norm{\nabla_z h}_{\bP_n ^{\prime}, p^*} = O(\beta_n + \alpha_n ^{1+k}).
\label{ineq:approx_upper_bound}
\end{align}

[Step 2] In this step, we establish a lower bound for the local worst-case risk $R_{\alpha_n,p} ^{\mathrm{worst}}(\bP_n, h)$. 
By the definition of the Wasserstein ball $\fM_{\alpha_n,p} (\bP_n)$, we have
\begin{align*}
&R_{\alpha_n,p} ^{\mathrm{worst}}(\bP_n, h) - R(\bP_n ^{\prime}, h) \\
&\geq \sup_{\tilde{z}_i \in \cZ} \left\{ \frac{1}{n} \sum_{i=1} ^n \{ h(\tilde{z}_i) - h(z_i ^{\prime} ) \} \mid \left( \frac{1}{n} \sum_{i=1} ^n \norm{\tilde{z}_i - z_i}^p \right)^{1/p} \leq \alpha_n \right\}.
\end{align*}
Again, the mean value theorem and the H\"older continuity assumption on $\nabla_z h$ give \begin{align*}
h(\tilde{z}_i) &= h(z_i ^{\prime}) + \langle \nabla_z h(c_i), \tilde{z}_i -z_i ^{\prime} \rangle \\
&= h(z_i ^{\prime}) + \langle \nabla_z h(z_i ^{\prime}), \tilde{z}_i -z_i ^{\prime} \rangle + \langle \nabla_z h(c_i) - \nabla_z h(z_i ^{\prime}), \tilde{z}_i -z_i ^{\prime} \rangle \\
&\geq h(z_i ^{\prime}) + \langle \nabla_z h(z_i ^{\prime}), \tilde{z}_i -z_i ^{\prime} \rangle 
- C_{\mathrm{H}}\norm{\tilde{z}_i - z_i ^{\prime}} ^{1+k} \\
&\geq h(z_i ^{\prime}) + \langle \nabla_z h(z_i ^{\prime}),  (\tilde{z}_i -z_i ) + (z_i -z_i ^{\prime}) \rangle
- C_{\mathrm{H},k} \left(\norm{\tilde{z}_i - z_i } ^{1+k} + \beta_n^{1+k} \right) \\
&\geq h(z_i ^{\prime}) + \langle \nabla_z h(z_i ^{\prime}),  \tilde{z}_i -z_i \rangle
- \norm{\nabla_z h(z_i ^{\prime})}_* \beta_n - C_{\mathrm{H},k} \left(\norm{\tilde{z}_i - z_i } ^{1+k} + \beta_n^{1+k} \right),
\end{align*}
where $c_i = t z_i + (1-t) \tilde{z}_i$ for some $t \in [0,1]$.
Thus, we have
\begin{align*}
&R_{\alpha_n,p} ^{\mathrm{worst}}(\bP_n, h) - R(\bP_n ^{\prime}, h) \\
\geq& -\beta_n( \norm{\nabla_z h}_{\bP_n ^{\prime},1} + C_{\mathrm{H},k} \beta_n ^k ) \\
&+ \sup_{\tilde{z}_i \in \cZ} \Big\{ \frac{1}{n} \sum_{i=1} ^n \{ \langle \nabla_z h(z_i ^{\prime}), \tilde{z}_i -z_i \rangle - C_{\mathrm{H},k} \norm{\tilde{z}_i - z_i} ^{1+k} \} \mid \left( \frac{1}{n} \sum_{i=1} ^n \norm{\tilde{z}_i - z_i}^p \right)^{1/p} \leq \alpha_n \Big\} \\
\geq& -\beta_n( \norm{\nabla_z h}_{\bP_n ^{\prime},1} + C_{\mathrm{H},k} \beta_n ^k ) \\
& + \sup_{\tilde{z}_i \in \cZ} \left\{ \frac{1}{n} \sum_{i=1} ^n \langle \nabla_z h(z_i ^{\prime}), \tilde{z}_i -z_i \rangle \mid \left( \frac{1}{n} \sum_{i=1} ^n \norm{\tilde{z}_i - z_i}^p \right)^{1/p} \leq \alpha_n \right\} \\
&- \sup_{\tilde{z}_i \in \cZ} \left\{ \frac{1}{n} \sum_{i=1} ^n C_{\mathrm{H},k} \norm{\tilde{z}_i - z_i} ^{1+k}  \mid \left( \frac{1}{n} \sum_{i=1} ^n \norm{\tilde{z}_i - z_i}^p \right)^{1/p} \leq \alpha_n \right\} \\
=:& -\beta_n \left( \norm{\nabla_z h}_{\bP_n ^{\prime},1} + C_{\mathrm{H},k} \beta_n ^k \right) + S_1 - S_2.
\end{align*}
As for the term $S_1$, by the definition of the dual norm we have
\begin{align*}
S_1 &\leq \sup_{\tilde{z}_i \in \cZ} \left\{ \frac{1}{n} \sum_{i=1} ^n \norm{\nabla_z h(z_i ^{\prime})}_*  \norm{\tilde{z}_i -z_i} \mid \left( \frac{1}{n} \sum_{i=1} ^n \norm{\tilde{z}_i - z_i}^p \right)^{1/p} \leq \alpha_n \right\},
\end{align*}
and by the H{\"o}lder inequality, 
\begin{align*}
\frac{1}{n} \sum_{i=1} ^n \norm{\nabla_z h(z_i ^{\prime})}_*  \norm{\tilde{z}_i -z_i} &\leq \left( \frac{1}{n} \sum_{i=1} ^n \norm{\nabla_z h(z_i ^{\prime})}_* ^{p^*} \right)^{1/p^*} \left( \frac{1}{n} \sum_{i=1} ^n \norm{\tilde{z}_i - z_i}^p \right)^{1/p} \\
&\leq \alpha_n \norm{\nabla_z h }_{\bP_n ^{\prime}, p^*},
\end{align*}
where the inequalities hold with equalities when for all $i \in [n]$
\begin{align*}
\norm{\tilde{z}_i - z_i} = \alpha_n \left( \frac{ \norm{\nabla_z h (z_i ^{\prime})}_* ^{p^*} }{ \frac{1}{n} \sum_{j=1} ^n \norm{\nabla_z h (z_j ^{\prime})}_* ^{p^*} } \right)^{1/p}.
\end{align*}
Here,
\begin{align*}
\alpha_n \left( \frac{ \norm{\nabla_z h (z_i ^{\prime})}_* ^{p^*} }{ \frac{1}{n} \sum_{j=1} ^n \norm{\nabla_z h (z_j ^{\prime})}_* ^{p^*} } \right)^{1/p} =  \alpha_n \left( \frac{ \norm{\nabla_z h (z_i ^{\prime})}_*  }{ \norm{ \nabla_z h }_{\bP_n ^{\prime}, p^*} } \right)^{p^*/p} \leq \alpha_n \left( \frac{ \norm{\nabla_z h (z_i ^{\prime})}_*  }{ \norm{ \nabla_z h }_{\bP_n ^{\prime}, 1} } \right)^{p^*/p}.
\end{align*}

Since $\alpha_n$ vanishes and $\cZ$ is an open set, $\tilde{z}_i \in \cZ$ if the term $\frac{ \norm{\nabla_z h (z_i ^{\prime})}_*  }{ \norm{ \nabla_z h }_{\bP_n ^{\prime}, 1} }$ is bounded. 
That is, the boundedness of $\frac{ \norm{\nabla_z h (z_i ^{\prime})}_*  }{ \norm{ \nabla_z h }_{\bP_n ^{\prime}, 1} }$ is a sufficient condition to achieve $S_1 =\alpha_n \norm{\nabla_z h }_{\bP_n ^{\prime}, p^*}$.
It is noteworthy that the numerator $\norm{\nabla_z h (z_i ^{\prime})}_*$ is bounded by $L_{\nabla}$, and due to the local perturbation, we have
\begin{align*}
\norm{ \nabla_z h(z_i ^{\prime})}_* &\geq \norm{ \nabla_z h(z_i ^{\prime})}_* - \norm{ \nabla_z h(z_i ) - \nabla_z h(z_i ^{\prime})}_* \\
&\geq \norm{ \nabla_z h(z_i )}_* - C_{\mathrm{H}} \norm{z_i ^{\prime} -z_i}^{1+k} \\
&\geq \norm{ \nabla_z h(z_i )}_* - C_{\mathrm{H}} \beta_n^{1+k}.
\end{align*}
Thus it is enough to show that the denominator $\norm{ \nabla_z h }_{\bP_n, 1}$ has a lower bound.

By the assumption $\bE_{\mathrm{data}}(\norm{\nabla_z h}_*) \geq C_{\nabla}$ and the fact $\norm{\nabla_z h (z)}_* \leq L_{\nabla}$ for all $z \in \mathrm{Conv}(\cZ) + \mathcal{B}(M)$, the McDiarmid inequality \citep[pages 136-137]{devroye2013} implies that for a fixed $\delta >0$, the following holds with probability at least $1-\delta$.
\begin{align}
\norm{ \nabla_z h }_{\bP_n , 1} \geq \bE_{\mathrm{data}}(\norm{\nabla_z h}_*)  - L_{\nabla} \sqrt{\frac{2}{n}\log(\frac{1}{\delta})}.
\label{ineq:lower_bound_mcdiarmid}
\end{align}
Therefore, for a large enough $n$, $\norm{ \nabla_z h }_{\bP_n, 1}$ is strictly greater than zero with high probability, and this implies that $S_1 = \alpha_n \norm{\nabla_z h }_{\bP_n ^{\prime}, p^*}$ with high probability.

As for the term $S_2$, we note the fact $ ( \frac{1}{n} \sum_{i=1} ^n \norm{\tilde{z}_i - z_i} ^{1+k} )^{ \frac{1}{1+k}} \leq \left( \frac{1}{n} \sum_{i=1} ^n \norm{\tilde{z}_i - z_i}^p \right)^{1/p}$ as $p > 1+k$.  
Since the equality holds when $\norm{\tilde{z}_i - z_i} = \alpha_n$ for all $i \in [n]$, we have 
\begin{align*}
\sup_{\tilde{z}_i \in \cZ} \left\{ \frac{1}{n} \sum_{i=1} ^n C_{\mathrm{H},k} \norm{\tilde{z}_i - z_i} ^{1+k}  \mid \left( \frac{1}{n} \sum_{i=1} ^n \norm{\tilde{z}_i - z_i}^p \right)^{1/p} \leq \alpha_n \right\} \leq C_{\mathrm{H},k} \alpha_n ^{1+k}.
\end{align*}

Thus, combining the terms $S_1$ and $S_2$ shows that for a large enough $n$ and a fixed $\delta >0$, the following holds with probability at least $1-\delta$.
\begin{align}
R_{\alpha_n,p} ^{\mathrm{worst}}(\bP_n, h) - R(\bP_n ^{\prime}, h) - \alpha_n \norm{\nabla_z h }_{\bP_n ^{\prime}, p^*} \geq -\beta_n( \norm{\nabla_z h}_{\bP_n ^{\prime}, 1} + C_{\mathrm{H},k} \beta_n ^k )  - C_{\mathrm{H},k} \alpha_n ^{1+k}.
\label{ineq:approx_lower_bound}
\end{align}

[Step 3] By the inequalities \eqref{ineq:approx_upper_bound} and \eqref{ineq:approx_lower_bound}, we have the following.
\begin{align*}
\left| R(\bP_n ^{\prime}, h) +\alpha_n \norm{\nabla_z h }_{\bP_n ^{\prime}, p^*} - R_{\alpha_n,p} ^{\mathrm{worst}}(\bP_n, h) \right| = O_p (\beta_n + \alpha_n ^{1+k}).
\end{align*}
This concludes the proof.
\end{proof}

\begin{remark}
\label{remark:uniform_bound}
The inequality \eqref{ineq:lower_bound_mcdiarmid} shows that $\norm{ \nabla_z h }_{\bP_n, 1}$ has a lower bound with high probability. 
To appropriately use the result of Theorem \ref{thm:wdro_perturb_approx_sharp_bound} to Theorems \ref{thm:wdro_perturb_excess_worst_risk} and \ref{thm:wdro_perturb_excess_risk}, we need a uniform bound result of $\norm{ \nabla_z h }_{\bP_n, 1}$.
Note that the inequality \eqref{ineq:lower_bound_mcdiarmid} does not hold when the loss $h$ depends on data.
We use the same $\cH$ as in Theorems \ref{thm:wdro_perturb_excess_worst_risk} and \ref{thm:wdro_perturb_excess_risk} and give a uniform bound result in the following proposition.
\end{remark}

\begin{proposition}
Let $\cZ$ be an open and bounded subset of $\bR^d$.
For constants $C_{\mathrm{H}}, C_{\nabla}, L >0$, $k \in (0,1]$, and $M \geq \sup_{n \in \mathbb{N}} \beta_n$, we let $\cH$ be a uniformly bounded set of differentiable functions $h: \mathrm{Conv}(\cZ)+\mathcal{B}(M) \to \mathbb{R}$ such that its gradient $\nabla_z h$ is $(C_{\mathrm{H}}, k)$-H{\"o}lder continuous, $\bE_{\mathrm{data}}(\norm{\nabla_z h}_*) \geq C_{\nabla}$, and $\mLip(h) \leq L$.
Then, for $\delta > 0$ and a large enough $n$, the following holds with probability at least $1-\delta$.
\begin{align*}
\norm{ \nabla_z h }_{\bP_n, 1} \geq \bE_{\mathrm{data}}(\norm{\nabla_z h}_*) - 2 \sqrt{2} \left( LC_{\mathrm{H},k, 2} + \frac{k}{d L C_{\mathrm{H},k, 2}} \right) n^{-\frac{k}{2k+d}} - L \sqrt{\frac{2}{n}\log(\frac{2}{\delta})},
\end{align*}
for some constant $C_{\mathrm{H},k, 2} >0$.
\label{prop:uniform_bound}
\end{proposition}

\begin{proof}
By the McDiarmid inequality \citep[pages 136-137]{devroye2013} and symmetrization arguments \citep[Lemma 2.3.1]{van1996}, for $\delta > 0$, the following holds with probability at least $1-\delta$.
\begin{align*}
\sup_{h \in \cH} \left| \norm{ \nabla_z h }_{\bP_n, 1} -  \bE_{\mathrm{data}}(\norm{\nabla_z h}_*) \right| \leq 2 \mathfrak{R}_n (\nabla\tilde{\cH}) + L \sqrt{\frac{2}{n}\log(\frac{2}{\delta})},
\end{align*}
where $\nabla\tilde{\cH} := \{ \norm{\nabla_z h }_* \mid h \in \cH \}$.
By the assumption $\bE_{\mathrm{data}}(\norm{\nabla_z h}_*) \geq C_{\nabla}$ and the fact that $L\sqrt{\frac{2}{n}\log(\frac{2}{\delta})}$ converges to zero as $n$ increases, $\norm{ \nabla_z h }_{\bP_n, 1}$ is strictly greater than zero if $\mathfrak{R}_n (\nabla\tilde{\cH})$ vanishes.
Therefore, it is enough to show that $\mathfrak{R}_n (\nabla\tilde{\cH})$ vanishes.

We denote a set of $(C_{\mathrm{H}}, k)$-H{\"o}lder continuous functions by $\cG_{\mathrm{H},k} := \{ g :\cZ \to \bR \mid g $ is $(C_{\mathrm{H}}, k)$-H{\"o}lder continuous and $\norm{g}_{\infty} \leq L$.\}. 
Then for all $\norm{\nabla_z \tilde{h} }_* \in \nabla\tilde{\cH}$, $\norm{\nabla_z \tilde{h} }_*$ is $(C_{\mathrm{H}}, k)$-H{\"o}lder continuous because 
\begin{align*}
    \left| \norm{\nabla_z \tilde{h} (z_{[1]}) }_* - \norm{\nabla_z \tilde{h} (z_{[2]}) }_* \right| &\leq \norm{\nabla_z \tilde{h} (z_{[1]}) - \nabla_z \tilde{h} (z_{[2]}) }_* \\
    &\leq C_{\mathrm{H}} \norm{z_{[1]} -z_{[2]}}^k,
\end{align*}
for all $z_{[1]}, z_{[2]} \in \mathrm{Conv}(\cZ)+\mathcal{B}(M)$. 
Further, because of the differentiability and Lipschitz continuity of $\tilde{h} \in \cH$, we have $\norm{ \norm{\nabla_z \tilde{h} }_* }_{\infty} \leq L$.
Thus $\nabla\tilde{\cH} \subseteq \cG_{\mathrm{H},k}$, which implies $\mathfrak{R}_n (\nabla\tilde{\cH}) \leq \mathfrak{R}_n (\cG_{\mathrm{H},k})$.

For $u>0$, let $N_u := \cN(u, \cG_{\mathrm{H},k}, \norm{\cdot}_{\infty})$ be the $u$-covering number of $\cG_{\mathrm{H},k}$ with respect to $\norm{\cdot}_{\infty}$ and let $\tilde{\cG}_u := \{ \tilde{g}_1, \dots, \tilde{g}_{N_u} \}$ be the corresponding $u$-cover.
For a set $\{\sigma_i\}_{i=1} ^{n}$ of independent Rademacher random variables, for some $j \in [N_u]$,
\begin{align*}
\frac{1}{n} | \sum_{i=1} ^{n} \sigma_i g(z_i) |  & \leq  \frac{1}{n} |\sum_{i=1} ^{n} \sigma_i \tilde{g}_j (z_i) | + \frac{1}{n} | \sum_{i=1} ^{n} \sigma_i (g(z_i) - \tilde{g}_j (z_i))| \\
&\leq \frac{1}{n} |\sum_{i=1} ^{n} \sigma_i \tilde{g}_j (z_i) | + u.
\end{align*}
The second inequality is due to the Cauchy–Schwarz inequality.
Then by the Massart\rq{}s lemma for a bounded and finite function space, we have 
\begin{align*}
\sup_{g \in \cG_{\mathrm{H},k}} \frac{1}{n} | \sum_{i=1} ^{n} \sigma_i g(z_i) | \leq \sup_{ \tilde{g} \in \tilde{\cG}_u } \frac{1}{n} |\sum_{i=1} ^{n} \sigma_i \tilde{g} (z_i) | + u \leq L \sqrt{ \frac{2 \log N_u}{n}} + u.
\end{align*}
Therefore, 
\begin{align*}
\mathfrak{R}_n (\cG_{\mathrm{H},k}) &\leq \inf_{u>0} \left\{ u + L \sqrt{\frac{ 2 \log \cN(u, \cG_{\mathrm{H},k}, \norm{\cdot}_{\infty}) }{n}} \right\} \\
&\leq \inf_{u>0} \left\{ u + L\sqrt{2(1+C_{\mathrm{H},k, 2})} \sqrt{\frac{ u^{-d/k} }{n}} \right\}\\
&= \left(L\sqrt{2(1+C_{\mathrm{H},k, 2})} \right)^{\frac{2k}{2k+d}} \left( \left(\frac{d}{2k} \right)^{\frac{2k}{2k+d}}+ \left(\frac{d}{2k} \right)^{-\frac{d}{2k+d}} \right) n^{-\frac{k}{2k+d}},
\end{align*}
for some constant $C_{\mathrm{H},k, 2}>0$.
Here, the second inequality is due to \citet[Theorem 2]{lorentz1962}:
\begin{align*}
C_{\mathrm{H},k, 1} \leq \lim_{u \to 0} \frac{\log \cN(u, \cG_{\mathrm{H},k}, \norm{\cdot}_{\infty})}{u^{-d/k}} \leq C_{\mathrm{H},k, 2}, 
\end{align*} 
for some constant $C_{\mathrm{H},k, 1} >0$.\footnote{\citet[Theorem 2]{lorentz1962} considers the uniform norm $\norm{\cdot}_{\infty}$ on $\cZ$, but any norm gives the same conclusion because any two norms are equivalent on the finite dimensional space $\bR^{d}$.}
Therefore, $\mathfrak{R}_n (\cG_{\mathrm{H},k})$ vanishes with high probability. 
\end{proof}

\subsection{Proof of Theorem \ref{thm:wdro_perturb_excess_worst_risk}}
\label{s:wdro_perturb_excess_worst_risk_proof}
\begin{proof}
Let $h_{\alpha_n,p,\cH} ^{\mathrm{worst}} = \mathrm{argmin}_{h \in \cH} R_{\alpha_n,p} ^{\mathrm{worst}}(\bP_{\mathrm{data}}, h)$.
Since $\cZ$ is bounded and $\cH$ is uniformly bounded, there exist constants $D_{\cZ}$ and $C_{\cH, \infty}$ such that $\sup_{z_1, z_2 \in \cZ} \norm{z_1 - z_2} \leq D_{\cZ}$ and $\sup_{h\in \cH} \sup_{z \in \cZ} | h(z)| \leq C_{\cH, \infty}$, respectively. 
As for the outline, we decompose an excess risk as follows. 
\begin{align*}
R_{\alpha_n,p} ^{\mathrm{worst}}(\bP_{\mathrm{data}}, \hat{h}_{(\alpha_n, \beta_n),p} ^{\mathrm{prop}}) - R_{\alpha_n,p} ^{\mathrm{worst}}(\bP_{\mathrm{data}}, h_{\alpha_n,p,\cH} ^{\mathrm{worst}}) &=  \underbrace{R_{\alpha_n,p} ^{\mathrm{worst}}(\bP_{\mathrm{data}}, \hat{h}_{(\alpha_n, \beta_n),p} ^{\mathrm{prop}}) - R_{\alpha_n,p} ^{\mathrm{worst}}(\bP_n, \hat{h}_{(\alpha_n, \beta_n),p} ^{\mathrm{prop}})}_{\mathrm{(T1)}}  \\
&+ \underbrace{R_{\alpha_n,p} ^{\mathrm{worst}}(\bP_n, \hat{h}_{(\alpha_n, \beta_n),p} ^{\mathrm{prop}}) - R_{\alpha_n,p} ^{\mathrm{worst}}(\bP_n, \hat{h}_{\alpha_n, p} ^{\mathrm{worst}})}_{\mathrm{(T2)}}  \\
&+ \underbrace{R_{\alpha_n,p} ^{\mathrm{worst}}(\bP_n, \hat{h}_{\alpha_n, p} ^{\mathrm{worst}}) - R_{\alpha_n,p} ^{\mathrm{worst}}(\bP_n, h_{\alpha_n,p,\cH} ^{\mathrm{worst}})}_{\mathrm{(T3)}}  \\
&+ \underbrace{R_{\alpha_n,p} ^{\mathrm{worst}}(\bP_n, h_{\alpha_n,p,\cH} ^{\mathrm{worst}}) - R_{\alpha_n,p} ^{\mathrm{worst}}(\bP_{\mathrm{data}}, h_{\alpha_n,p,\cH} ^{\mathrm{worst}})}_{\mathrm{(T4)}}.
\end{align*} 

As for the term $\mathrm{(T3)}$, by the definition of $\hat{h}_{\alpha_n, p} ^{\mathrm{worst}}$, 
\begin{align*}
\mathrm{(T3)} = R_{\alpha_n,p} ^{\mathrm{worst}}(\bP_n, \hat{h}_{\alpha_n, p} ^{\mathrm{worst}}) - R_{\alpha_n,p} ^{\mathrm{worst}}(\bP_n, h_{\alpha_n,p,\cH} ^{\mathrm{worst}}) \leq 0.
\end{align*} 

[Step 1] In this step, we obtain an upper bound of the term $\mathrm{(T2)}$. 
By Theorem \ref{thm:wdro_perturb_approx_sharp_bound}, for any fixed $\delta >0$, there exists finite constants $\tilde{M}_1 >0, \tilde{N}_1 \in \mathbb{N}$ such that the following holds with probability at least $1-\delta/2$.\footnote{We refer Remark \ref{remark:uniform_bound} and Proposition \ref{prop:uniform_bound}.}
\begin{align}\label{eqn:M1}
\frac{ | R_{\alpha_n,p} ^{\mathrm{worst}}(\bP_n, \hat{h}_{(\alpha_n, \beta_n),p} ^{\mathrm{prop}}) - R_{(\alpha_n, \beta_n), p} ^{\mathrm{prop}} (\bP_{n}, \hat{h}_{(\alpha_n, \beta_n),p} ^{\mathrm{prop}}) | }{\beta_n + \alpha_n ^{1+k}} &\leq \tilde{M}_1, 
\end{align}
for any $n \geq \tilde{N}_1$. 
Similarly, there exists finite constants $\tilde{M}_2 >0, \tilde{N}_2 \in \mathbb{N}$ such that the following holds with probability at least $1-\delta/2$.
\begin{align}\label{eqn:M2}
\frac{ | R_{\alpha_n,p} ^{\mathrm{worst}}(\bP_n, \hat{h}_{\alpha_n, p} ^{\mathrm{worst}}) - R_{(\alpha_n, \beta_n), p} ^{\mathrm{prop}} (\bP_{n}, \hat{h}_{\alpha_n, p} ^{\mathrm{worst}}) | }{\beta_n + \alpha_n ^{1+k}} &\leq \tilde{M}_2,
\end{align}
for any $n \geq \tilde{N}_2$.
Choose $\varepsilon_n > 0$ so that $\varepsilon_n = \Theta(\log(n)(\beta_n + \alpha_n ^{1+k}))$.\footnote{For positive sequences $(a_n)$ and $(b_n)$, $b_n = \Theta(a_n)$ indicates that there exist $C_1 >0, C_2>0, n_0 \in \mathbb{N}$ such that $C_1 a_n \leq b_n \leq C_2 a_n$ for all $n \geq n_0$.}
Then there exists  $\tilde{N} \ge \max\{\tilde{N}_1, \tilde{N}_2\}$ such that for all $n \ge \tilde{N}$, we have
$\varepsilon_n - (\tilde{M}_1  +\tilde{M}_2 )  (\beta_n + \alpha_n ^{1+k}) > 0$.
Fix such $n$. Under the product of the above two events \eqref{eqn:M1} and \eqref{eqn:M2}, assume that $R_{\alpha_n,p} ^{\mathrm{worst}}(\bP_n, \hat{h}_{(\alpha_n, \beta_n),p} ^{\mathrm{prop}}) > R_{\alpha_n,p} ^{\mathrm{worst}}(\bP_n, \hat{h}_{\alpha_n, p} ^{\mathrm{worst}}) + \varepsilon_n$. Then
\begin{align*}
R_{(\alpha_n, \beta_n), p} ^{\mathrm{prop}} (\bP_{n}, \hat{h}_{(\alpha_n, \beta_n),p} ^{\mathrm{prop}}) &\geq R_{\alpha_n,p} ^{\mathrm{worst}}(\bP_n, \hat{h}_{(\alpha_n, \beta_n),p} ^{\mathrm{prop}}) - \tilde{M}_1 (\beta_n + \alpha_n ^{1+k}) \\
&> R_{\alpha_n,p} ^{\mathrm{worst}}(\bP_n, \hat{h}_{\alpha_n, p} ^{\mathrm{worst}}) + \varepsilon_n - \tilde{M}_1 (\beta_n + \alpha_n ^{1+k}) \\
&\geq R_{(\alpha_n, \beta_n), p} ^{\mathrm{prop}} (\bP_{n}, \hat{h}_{\alpha_n, p} ^{\mathrm{worst}}) + \varepsilon_n - (\tilde{M}_1  +\tilde{M}_2 )  (\beta_n + \alpha_n ^{1+k}) \\
& > R_{(\alpha_n, \beta_n), p} ^{\mathrm{prop}} (\bP_{n}, \hat{h}_{\alpha_n, p} ^{\mathrm{worst}}),
\end{align*}
which contradicts the definition of $\hat{h}_{(\alpha_n, \beta_n),p} ^{\mathrm{prop}}$. 
Thus, with probability at least $1-\delta$, we have
\begin{align*}
\mathrm{(T2)} =  R_{\alpha_n,p} ^{\mathrm{worst}}(\bP_n, \hat{h}_{(\alpha_n, \beta_n),p} ^{\mathrm{prop}}) - R_{\alpha_n,p} ^{\mathrm{worst}}(\bP_n, \hat{h}_{\alpha_n, p} ^{\mathrm{worst}}) 
\leq \varepsilon_n = \Theta(\log(n)(\beta_n + \alpha_n ^{1+k})).
\end{align*}
for sufficiently large $n$, or
\begin{align}
\mathrm{(T2)} =  R_{\alpha_n,p} ^{\mathrm{worst}}(\bP_n, \hat{h}_{(\alpha_n, \beta_n),p} ^{\mathrm{prop}}) - R_{\alpha_n,p} ^{\mathrm{worst}}(\bP_n, \hat{h}_{\alpha_n, p} ^{\mathrm{worst}}) = O( \log(n)(\beta_n + \alpha_n ^{1+k})).
\label{ineq:wost:step1_result}
\end{align}

[Step 2] This step is based on proof of \citet[Theorem 3]{lee2018}.
As for the term $\mathrm{(T1)}$, by the inequality (C.4) and Lemma 5 of \citet{lee2018}, we have 
\begin{align}
\mathrm{(T1)} = R_{\alpha_n,p} ^{\mathrm{worst}}(\bP_{\mathrm{data}}, \hat{h}_{(\alpha_n, \beta_n),p} ^{\mathrm{prop}}) - R_{\alpha_n,p} ^{\mathrm{worst}}(\bP_n, \hat{h}_{(\alpha_n, \beta_n),p} ^{\mathrm{prop}}) \leq \frac{48 \mathfrak{C}(\cH)}{\sqrt{n}} + \frac{48 L D_{\cZ}^{p}}{\sqrt{n} \alpha_n ^{p-1}} + C_{\cH, \infty} \sqrt{\frac{2}{n}\log(\frac{2}{\delta})},
\label{ineq:wost:step2_result}
\end{align} 
with probability at least $1-\delta/2$.

As for the term $\mathrm{(T4)}$, by the inequality (C.5) of \citet{lee2018}, the following holds with probability at least $1-\delta/2$.
\begin{align}
\mathrm{(T4)} = R_{\alpha_n,p} ^{\mathrm{worst}}(\bP_n, h_{\alpha_n,p,\cH} ^{\mathrm{worst}}) - R_{\alpha_n,p} ^{\mathrm{worst}}(\bP_{\mathrm{data}}, h_{\alpha_n,p,\cH} ^{\mathrm{worst}}) \leq  C_{\cH, \infty} \sqrt{\frac{2}{n}\log(\frac{2}{\delta})}.
\label{ineq:wost:step3_result}
\end{align}
Therefore, by combining all the inequalities \eqref{ineq:wost:step1_result}, \eqref{ineq:wost:step2_result}, and \eqref{ineq:wost:step3_result}, the following holds with probability at least $1-2\delta$,
\begin{align*}
&R_{\alpha_n,p} ^{\mathrm{worst}}(\bP_{\mathrm{data}}, \hat{h}_{(\alpha_n, \beta_n),p} ^{\mathrm{prop}}) - R_{\alpha_n,p} ^{\mathrm{worst}}(\bP_{\mathrm{data}}, h_{\alpha_n,p,\cH} ^{\mathrm{worst}}) \\
&\leq \frac{48 \mathfrak{C}(\cH)}{\sqrt{n}} + \frac{48 L D_{\cZ}^{p}}{\sqrt{n} \alpha_n ^{p-1}} + 2 C_{\cH, \infty} \sqrt{\frac{2}{n}\log(\frac{2}{\delta})} + O (\log(n)(\beta_n + \alpha_n ^{1+k})) \\
&= O( n^{-1/2}( \mathfrak{C}(\cH) + \alpha_n ^{1-p}) + \log(n) (\beta_n + \alpha_n ^{1+k}) ).
\end{align*}
This concludes the proof.
\end{proof}

\subsection{Proof of Theorem \ref{thm:wdro_perturb_excess_risk}}
\label{s:wdro_perturb_excess_risk_proof}

\begin{proof}[Proof of Theorem \ref{thm:wdro_perturb_excess_risk}]
Let $h_{\cH} = \mathrm{argmin}_{h \in \cH} R(\bP_{\mathrm{data}}, h)$.
Since $\cH$ is uniformly bounded, there exists a constant $C_{\cH, \infty}$ such that $\sup_{h\in \cH} \sup_{z \in \cZ} | h(z)| \leq C_{\cH, \infty}$.
Now decompose the excess risk as follows. 
\begin{align*}
R(\bP_{\mathrm{data}}, \hat{h}_{(\alpha_n, \beta_n),p} ^{\mathrm{prop}}) - R(\bP_{\mathrm{data}}, h_{\cH}) &=  \underbrace{R(\bP_{\mathrm{data}}, \hat{h}_{(\alpha_n, \beta_n),p} ^{\mathrm{prop}}) - R(\bP_{n}, \hat{h}_{(\alpha_n, \beta_n),p} ^{\mathrm{prop}})}_{\mathrm{(T1)}}  \\
&+ \underbrace{R(\bP_{n}, \hat{h}_{(\alpha_n, \beta_n),p} ^{\mathrm{prop}}) - R_{(\alpha_n, \beta_n), p} ^{\mathrm{prop}} (\bP_{n}, \hat{h}_{(\alpha_n, \beta_n),p} ^{\mathrm{prop}})}_{\mathrm{(T2)}}  \\
&+ \underbrace{R_{(\alpha_n, \beta_n), p} ^{\mathrm{prop}} (\bP_{n}, \hat{h}_{(\alpha_n, \beta_n),p} ^{\mathrm{prop}}) - R_{\alpha_n,p} ^{\mathrm{worst}}(\bP_n, \hat{h}_{(\alpha_n, \beta_n),p} ^{\mathrm{prop}})}_{\mathrm{(T3)}}  \\
&+ \underbrace{R_{\alpha_n,p} ^{\mathrm{worst}}(\bP_n, \hat{h}_{(\alpha_n, \beta_n),p} ^{\mathrm{prop}}) - R_{\alpha_n,p} ^{\mathrm{worst}}(\bP_n, \hat{h}_{\alpha_n, p} ^{\mathrm{worst}})}_{\mathrm{(T4)}}  \\
&+ \underbrace{ R_{\alpha_n,p} ^{\mathrm{worst}}(\bP_n, \hat{h}_{\alpha_n, p} ^{\mathrm{worst}}) - R(\bP_{n}, \hat{h}_n ^{\mathrm{ERM}} )}_{\mathrm{(T5)}}  \\
&+ \underbrace{R(\bP_{n}, \hat{h}_n ^{\mathrm{ERM}}) - R(\bP_{\mathrm{data}}, h_{\cH})}_{\mathrm{(T6)}}.
\end{align*} 

[Step 1]
In this step, we obtain an upper bound of the term $\mathrm{(T5)}$.
For all $h \in \cH$ and small enough $\alpha_n$, we have
\begin{align}
R_{\alpha_n,p} ^{\mathrm{worst}}(\bP_n, h) &\leq R(\bP_{n}, h) + \mLip(h) \alpha_n \leq R(\bP_n, h) + L \alpha_n
\label{ineq:step2_1}.
\end{align}
The first inequality is due to the inequality \eqref{ineq:diff_worst_emp}, the second inequality is due to the assumption.
Applying the infimum operator to the inequality \eqref{ineq:step2_1} gives 
\begin{align*}
R_{\alpha_n,p} ^{\mathrm{worst}} (\bP_n, \hat{h}_{\alpha_n, p} ^{\mathrm{worst}}) &\leq R(\bP_{n}, \hat{h}_n ^{\mathrm{ERM}}) + L \alpha_n = R(\bP_{n}, \hat{h}_n ^{\mathrm{ERM}}) + O(\alpha_n). 
\end{align*}
Therefore,
\begin{align}
\mathrm{(T5)} = R_{\alpha_n,p} ^{\mathrm{worst}}(\bP_n, \hat{h}_{\alpha_n, p} ^{\mathrm{worst}}) - R(\bP_{n}, \hat{h}_n ^{\mathrm{ERM}}) = O(\alpha_n). \label{ineq:step2_result}
\end{align}

[Step 2]
In this step, we obtain an upper bound for the terms (T2), (T3), and (T4). 
For any fixed $\delta>0$, the following holds with probability at least $1-\delta$.
\begin{align*}
R(\bP_n, h) &\leq R_{\alpha_n,p} ^{\mathrm{worst}}(\bP_n, h) \leq R_{(\alpha_n, \beta_n), p} ^{\mathrm{prop}} (\bP_{n}, h) + O(\beta_n + \alpha_n ^{1+k}).
\end{align*}
The first inequality is due to $\bP_n \in \fM_{\alpha_n,p} (\bP_n)$ and the second inequality is due to Theorem \ref{thm:wdro_perturb_approx_sharp_bound}. Thus,
\begin{align*}
\mathrm{(T2)} = R(\bP_n, \hat{h}_{(\alpha_n, \beta_n),p} ^{\mathrm{prop}}) - R_{(\alpha_n, \beta_n), p} ^{\mathrm{prop}} (\bP_{n}, \hat{h}_{(\alpha_n, \beta_n),p} ^{\mathrm{prop}}) = O_p(\beta_n + \alpha_n ^{1+k}).
\end{align*}
As for the term (T3), by Theorem \ref{thm:wdro_perturb_approx_sharp_bound}, we have
\begin{align*}
\mathrm{(T3)} &= R_{(\alpha_n, \beta_n), p} ^{\mathrm{prop}} (\bP_{n}, \hat{h}_{(\alpha_n, \beta_n),p} ^{\mathrm{prop}}) - R_{\alpha_n,p} ^{\mathrm{worst}}(\bP_n, \hat{h}_{(\alpha_n, \beta_n),p} ^{\mathrm{prop}}) = O_p(\beta_n + \alpha_n ^{1+k}).
\end{align*}
As for the term (T4), the inequality \eqref{ineq:wost:step1_result} gives 
\begin{align*}
\mathrm{(T4)} =  R_{\alpha_n,p} ^{\mathrm{worst}}(\bP_n, \hat{h}_{(\alpha_n, \beta_n),p} ^{\mathrm{prop}}) - R_{\alpha_n,p} ^{\mathrm{worst}}(\bP_n, \hat{h}_{\alpha_n, p} ^{\mathrm{worst}}) = O_p( \log(n)(\beta_n + \alpha_n ^{1+k})). 
\end{align*}
Therefore, 
\begin{align}
\mathrm{(T2)} + \mathrm{(T3)} + \mathrm{(T4)} = O_p( \log(n)(\beta_n + \alpha_n ^{1+k})).
\label{ineq:step3_result}
\end{align}

[Step 3]
In this step, we obtain an upper bound for the terms (T1) and (T6).
Note that the term $\mathrm{(T1)}$ is bounded by $\sup_{h \in \mathcal{H}} | R(\bP_{n}, h) - R(\bP_{\mathrm{data}}, h) |$.
As for the term $\mathrm{(T6)}$, we have 
\begin{align*}
R(\bP_{n}, \hat{h}_n ^{\mathrm{ERM}}) - R(\bP_{\mathrm{data}}, h_{\cH}) &=  R(\bP_{n}, \hat{h}_n ^{\mathrm{ERM}}) - R(\bP_{n}, h_{\cH})+ R(\bP_{n}, h_{\cH}) - R(\bP_{\mathrm{data}}, h_{\cH}) \\
&\leq 0 + R(\bP_{n}, h_{\cH}) - R(\bP_{\mathrm{data}}, h_{\cH}) \\
&\leq \sup_{h \in \mathcal{H}} | R(\bP_{n}, h) - R(\bP_{\mathrm{data}}, h) |.
\end{align*}
The first inequality is due to the definition of $\hat{h}_n ^{\mathrm{ERM}}$.
Thus, the sum of the terms $\mathrm{(T1)}$ and $\mathrm{(T6)}$ is bounded by $2\sup_{h \in \mathcal{H}} | R(\bP_{n}, h) - R(\bP_{\mathrm{data}}, h) |$.
The McDiarmid inequality  \citep[pages 136-137]{devroye2013} and symmetrization arguments \citep[Lemma 2.3.1]{van1996} provide
\begin{align}
&\sup_{h \in \mathcal{H}} | R(\bP_{n}, h) - R(\bP_{\mathrm{data}}, h) | \leq 2 \mathfrak{R}_n (\mathcal{H}) + C_{\cH, \infty} \sqrt{\frac{2}{n}\log(\frac{2}{\delta})} , \label{ineq:step4_result}
\end{align}
with probability at least $1-\delta$.

Lastly, by aggregating the inequalities \eqref{ineq:step2_result}, \eqref{ineq:step3_result} and \eqref{ineq:step4_result}, 
\begin{align*}
&R(\bP_{\mathrm{data}}, \hat{h}_{(\alpha_n, \beta_n),p} ^{\mathrm{prop}}) - \inf_{h \in \cH}R(\bP_{\mathrm{data}}, h) \\
&= O_p( \mathfrak{R}_n (\mathcal{H}) + n^{-1/2}+ \alpha_n+ \log(n)(\beta_n + \alpha_n ^{1+k}) ).
\end{align*}
This concludes the proof.
\end{proof}

\subsection{Details for Section \ref{s:example_bounds}}
We first define some notations. 
Let $\cX \subseteq [-1,1]^{d-1}$ and $\cY = \{ \pm 1\}$ be open sets with respect to the $\ell_2$-norm and the discrete norm $I(\cdot \neq 0)$, respectively.
We set $\cZ = \cX \times \cY$ and $\norm{(x,y)} = \norm{x}_{2} + 4I(y \neq 0)$.
Note that $\cX \times \cY$ is clearly open and bounded with respect to $\norm{(x,y)}$.
For a matrix $\tilde{\mathbf{A}} \in \bR^{\tilde{d}_1 \times \tilde{d}_2}$, its Frobenius norm is defined as $\norm{\tilde{\mathbf{A}}}_{\rm F} = \sqrt{\sum_{i=1} ^{\tilde{d}_2} \sum_{j=1} ^{\tilde{d}_1} \tilde{\mathbf{A}}_{ij} ^2 }$ and the matrix $\ell_p$-norm $\norm{\tilde{\mathbf{A}}}_p := \sup_{\norm{u}_p =1} \norm{\tilde{\mathbf{A}}u}_p$ for $p \in [1,\infty]$.
Now we define the space of deep neural networks.
For an integer $J$ and a set of integers $\mathbf{d} := \{d_0, \dots, d_{J} \}$ such that $d_0 =d-1$ and $d_J = 1$, we let $\cA=\{ \mathbf{A}_1, \dots, \mathbf{A}_{J} \}$ be $J$ weight matrices such that $\mathbf{A}_i \in \bR^{d_i \times d_{i-1}}$.
For a constant $\gamma > 0$ and a set of positive constants $\mathbf{M} := \{M_1, \dots, M_J\}$, define
\begin{align*}
\cF_{\mathbf{d}, \mathbf{M}, \gamma} ^{\cX \times \cY} := \{ yf(x) = y\phi_J(\mathbf{A}_{J} \phi_{J-1} (\mathbf{A}_{J-1} \dots \phi_1 (\mathbf{A}_1 x) \dots)) \mid \norm{\mathbf{A}_i}_{\rm F} \leq M_i, i \in [J], \gamma \leq \prod_{i \in [J]}  \norm{\mathbf{A}_i}_2 \}\\
\cF_{\mathbf{d}, \mathbf{M}, \gamma} := \{ f(x) = \phi_J(\mathbf{A}_{J} \phi_{J-1} (\mathbf{A}_{J-1} \dots \phi_1 (\mathbf{A}_1 x) \dots)) \mid \norm{\mathbf{A}_i}_{\rm F} \leq M_i, i \in [J], \gamma \leq \prod_{i \in [J]}  \norm{\mathbf{A}_i}_2 \},
\end{align*}
where $\phi_i : \bR^{d_{i}} \to \bR^{d_{i}}$ is a 1-Lipschitz activation function and satisfies $\phi_i (\mathbf{0}_{d_i}) = \mathbf{0}_{d_i}$ for all $i \in [J]$, and $\mathbf{0}_{d_i}$ is the vector of $d_i$ zeros.
Note that we omit intercepts here for notational simplicity.
For $\phi_1, \dots, \phi_{J-1}$, we employ the hyperbolic tangent function and $\phi_J$ is the identity function.\footnote{We may employ other differentiable activation functions with Lipschitz constant less than or equal to one.
The differentiability of activation functions is required to satisfy the conditions of Theorem \ref{thm:wdro_approx_sharp_bound}.
However, it can be easily shown that this condition can be relaxed to hold only $\bP_{\mathrm{data}}$-almost surely, so that the ReLU function can be employed, by re-stating Theorem \ref{thm:wdro_approx_sharp_bound} with $\bP_{\mathrm{data}}$-almost sure conditions.
Here, for the sake of simplicity, we simply use the hyperbolic tangent function, which is differentiable.} 
Lastly, for a positive constants $s$, we define 
\begin{align}
\cF_{\mathbf{d}, \mathbf{M}, \gamma, s} ^{\cX \times \cY} := \{ yf(x) \in \cF_{\mathbf{d}, \mathbf{M}, \gamma} ^{\cX \times \cY} \mid \sum_{i \in [J]} \norm{\mathbf{A}_i}_0 \leq s \} \notag \\
\cF_{\mathbf{d}, \mathbf{M}, \gamma, s} := \{ f(x) \in \cF_{\mathbf{d}, \mathbf{M}, \gamma} \mid \sum_{i \in [J]} \norm{\mathbf{A}_i}_0 \leq s \},
\label{eq:dnn}
\end{align}
where $\norm{\mathbf{A}}_0$ is the number of non-zero entries of a matrix $\mathbf{A}$.
To this ends, we will set $\mathbf{M} = \mathbf{1}_J$, the vector of $J$ ones.

\begin{corollary}[A formal statement of Corollary \ref{cor:example_bounds}]
Let $\cF_{\mathbf{d}, \mathbf{1}_J, \gamma, s}$ be a set of sparse deep neural networks, defined in \eqref{eq:dnn}.
For some constant $C_{\nabla} >0$, let $\cH =\{ h(x,y) \mid h(x,y) = \log(1+\exp(-yf(x)))$ and $\bE_{\mathrm{data}}\left(\norm{\nabla_x f(x)}_2 \right) > C_{\nabla}$ for $f \in \cF_{\mathbf{d}, \mathbf{1}_J, \gamma, s} \}$.\footnote{The sufficient condition for $\bE_{\mathrm{data}}(\norm{\nabla_x f(x)}_2) > C_{\nabla}$ may not be obvious, but it is assumed to be held based on Figures 2 and 3.}
Then the excess worst-case risks of $\hat{h}_{\alpha_n,p} ^{\mathrm{prop}}$ and $\hat{h}_{n} ^{\mathrm{ERM}}$ are
\begin{align*}
\cE_{\alpha_n,p} ^{\mathrm{worst}}( \hat{h}_{\alpha_n,p} ^{\mathrm{prop}})&= O_p( n^{-1/2} \alpha_n ^{1-p} \vee \log(n) \alpha_n ^{1+k} ), \\
\cE_{\alpha_n,p} ^{\mathrm{worst}}( \hat{h}_{n} ^{\mathrm{ERM}}) &= O_p( n^{-1/2} \vee \alpha_n ).
\end{align*}
Furthermore, the excess risks of $\hat{h}_{\alpha_n,p} ^{\mathrm{prop}}$ and $\hat{h}_{n} ^{\mathrm{ERM}}$ are
\begin{align*}
\cE(\hat{h}_{\alpha_n,p} ^{\mathrm{prop}}) &= O_p( n^{-1/2} \vee \alpha_n \vee \log(n)( \alpha_n ^{1+k} ) ), \\
\cE( \hat{h}_{n} ^{\mathrm{ERM}}) &= O_p( n^{-1/2}).
\end{align*}
\label{cor:formal_example_bounds}
\end{corollary}

\begin{proof}

[Step 1]
Clearly, $\cX \times \cY$ is open and bounded. In addition, the domain $\cX$ is bounded and weights $\norm{\mathbf{A}_i}_{\rm F}$ are bounded for all $i \in [J]$, for all $f \in \cF_{\mathbf{d}, \mathbf{1}_J, \gamma, s}$, we have $\sup_{x \in \cX} |f(x)| \leq C_{\cF_{\mathbf{d}, \mathbf{1}_J, \gamma, s}}$ for some constant $C_{\cF_{\mathbf{d}, \mathbf{1}_J, \gamma, s}}>0$.
In short, $\cF_{\mathbf{d}, \mathbf{1}_J, \gamma, s}$ is uniformly bounded, and $\cH$ is uniformly bounded as well.
In addition, for all $f \in \cF_{\mathbf{d}, \mathbf{1}_J, \gamma, s}$, due to the differentiability of the hyperbolic tangent function, $f$ is twice continuously differentiable, and this implies that for all $h\in \cH$, $h$ is twice continuously differentiable.
Uniformly boundedness of $\cA$ and the boundedness of $\cZ$ implies that uniformly boundedness of $\norm{\nabla_z h}_{*}$ and the Frobenius norm of the Hessian matrix of $h$.
This provides existence of constants $C_{\mathrm{H}}$ and $L$ such that $\nabla_z h$ is $(C_{\mathrm{H}},1/2)$-H{\"o}lder continuous for all $h \in \cH$ and $\mLip(h) \leq L$.

Lastly, by the definition of the dual norm and the discrete norm, 
\begin{align}
\norm{\nabla_z h}_* = \sup_{\norm{u}\leq 1} 
\langle \nabla_z h, u \rangle = \sup_{\norm{s}_{2} \leq 1} \langle \nabla_x h, s \rangle = \norm{\nabla_x h}_{2}.
\label{eq:dual_norm_transform}
\end{align}
Since $\nabla_x h = \frac{\exp(-yf(x))}{1+\exp(-yf(x))} (-y) \nabla_x f(x)$, we have 
\begin{align*}
\norm{\nabla_x h}_2 = \left| \frac{1}{1+\exp(yf(x))} \right| \norm{\nabla_x f}_2 \geq \frac{1}{1+ \exp(C_{\cF_{\mathbf{d}, \mathbf{1}_J, \gamma, s}})} \norm{\nabla_x f}_2 .
\end{align*}
Therefore, all the conditions in Theorems \ref{thm:wdro_excess_worst_risk} and \ref{thm:wdro_excess_risk} are satisfied.

[Step 2]
Since $\ell_{\mathrm{log}}(z) := \log(1+\exp(-z))$ is continuously differentiable on $(-2B_{\cF}, 2B_{\cF})$, $\ell_{\mathrm{log}}(z)$ is Lipschitz continuous on $[-B_{\cF},B_{\cF}]$.
It implies that there exists a finite Lipschitz constant. 
Let $L_{\mathrm{log}}$ be a Lipschitz constant on $[-B_{\cF},B_{\cF}]$.
Due to Talagrand\rq{}s lemma \citep[Lemma 5.7]{mohri2018}, we have
\begin{align*}
\mathfrak{R}_n (\cH) = \mathfrak{R}_n (\ell_{\mathrm{log}} \circ \cF_{\mathbf{d}, \mathbf{1}_J, \gamma, s} ^{\cX \times \cY}) \leq L_{\mathrm{log}} \mathfrak{R}_n ( \cF_{\mathbf{d}, \mathbf{1}_J, \gamma, s} ^{\cX \times \cY}).
\end{align*}
Due to Lemma \ref{lem:golowich} below, we have
\begin{align*}
\mathfrak{R}_n  (\cF_{\mathbf{d}, \mathbf{1}_J, \gamma, s} ^{\cX \times \cY}) \leq \mathfrak{R}_n  (\cF_{\mathbf{d}, \mathbf{1}_J, \gamma} ^{\cX \times \cY})  = \mathfrak{R}_n (\cF_{\mathbf{d}, \mathbf{1}_J, \gamma}) \leq O(n^{-1/2}).
\end{align*}
The equality is due to for all $i \in [J]$,  $\sigma_i \stackrel{d}{=} \sigma_iy_i$ for the Rademacher random variables $\sigma_i$.
Therefore, by \citet[Theorem 11.3]{mohri2018} and Theorem \ref{thm:wdro_excess_risk}, $\cE( \hat{h}_{n} ^{\mathrm{ERM}}) = O_p( n^{-1/2})$ and $\cE(\hat{h}_{\alpha_n,p} ^{\mathrm{prop}}) = O_p( n^{-1/2} \vee \alpha_n \vee \log(n)( \alpha_n ^{1+k} ) )$ are obtained.

[Step 3]
Here we prove the excess worst-case risk bound for $\hat{h}_{n} ^{\mathrm{ERM}}$.
An essentially the same argument as \eqref{ineq:step2_1} yields that for all $h \in \cH$,
\begin{align}
R_{\alpha_n,p} ^{\mathrm{worst}}(\bP_{\mathrm{data}}, h) &\leq R(\bP_{\mathrm{data}}, h) + \mLip(h) \alpha_n \leq R(\bP_{\mathrm{data}}, h) + L \alpha_n,
\label{ineq:coro_step_1}
\end{align}
Applying the infimum operator on $R(\bP_{\mathrm{data}}, h)  \leq R_{\alpha_n,p} ^{\mathrm{worst}}(\bP_{\mathrm{data}}, h)$ gives
\begin{align}
\inf_{h \in \cH} R(\bP_{\mathrm{data}}, h) \leq \inf_{h \in \cH} R_{\alpha_n,p} ^{\mathrm{worst}}(\bP_{\mathrm{data}}, h).
\label{ineq:coro_step_2}
\end{align}
Therefore, the inequalities \eqref{ineq:coro_step_1} and \eqref{ineq:coro_step_2} give 
\begin{align*}
\cE_{\alpha_n,p} ^{\mathrm{worst}}(h) &= R_{\alpha_n,p} ^{\mathrm{worst}}(\bP_{\mathrm{data}}, h) - \inf_{h \in \cH} R_{\alpha_n,p} ^{\mathrm{worst}}(\bP_{\mathrm{data}}, h) \\
&\leq R(\bP_{\mathrm{data}}, h) + L \alpha_n - \inf_{h \in \cH} R(\bP_{\mathrm{data}}, h) \\
&= \cE(h) + L \alpha_n.
\end{align*}
By Theorem \ref{thm:wdro_excess_risk} we conclude that $\cE_{\alpha_n,p} ^{\mathrm{worst}}( \hat{h}_{n} ^{\mathrm{ERM}}) = O_p( n^{-1/2} \vee \alpha_n )$.

[Step 4] 
We now prove that $\cE_{\alpha_n,p} ^{\mathrm{worst}}( \hat{h}_{\alpha_n,p} ^{\mathrm{prop}}) = O_p( n^{-1/2} \alpha_n ^{1-p} \vee \log(n) \alpha_n ^{1+k} )$.
By Theorem \ref{thm:wdro_excess_worst_risk}, it is enough to show that $\mathfrak{C}(\cH) := \int_{0} ^{\infty} \sqrt{\log \cN(u, \cH, \norm{\cdot}_{\infty}) } du$ is finite.

For all $(x,y) \in \cZ$ and $f_1, f_2 \in \cF_{\mathbf{d}, \mathbf{1}_J, \gamma, s} $, we have
\begin{align*}
|\ell_{\mathrm{log}}(yf_1(x)) - \ell_{\mathrm{log}}(yf_2(x)) | &\leq L_{\mathrm{log}} | yf_1(x) -yf_2(x)| = L_{\mathrm{log}} |f_1(x)-f_2(x)|.
\end{align*}
Therefore, $\cN(u, \cH, \norm{\cdot}_{\infty}) \leq \cN( \frac{u}{L_{\mathrm{log}}}, \cF_{\mathbf{d}, \mathbf{1}_J, \gamma, s} , \norm{\cdot}_{\infty})$, and thus by Lemma \ref{lem:hieber} below we have
\begin{align*}
\log \cN( \frac{u}{L_{\mathrm{log}}}, \cF_{\mathbf{d}, \mathbf{1}_J, \gamma, s}, \norm{\cdot}_{\infty}) \leq (s+1) \log \left( \frac{2 J V^2 L_{\mathrm{log}}}{u} \right).
\end{align*}
Therefore, an integration by substitution gives
\begin{align*}
\int_{0} ^{\infty} \sqrt{\log \cN(u, \cH, \norm{\cdot}_{\infty}) } du &\leq \int_{0} ^{\infty} \sqrt{ \cN( \frac{u}{L_{\mathrm{log}}}, \cF_{\mathbf{d}, \mathbf{1}_J, \gamma, s}, \norm{\cdot}_{\infty})  } du  \\
&= \sqrt{(s+1)} \int_{0} ^{\infty} \sqrt{ \log \left( \frac{2 J V^2 L_{\mathrm{log}}}{u} \right)  } du \\
&= \sqrt{(s+1)} \int_{0} ^{2 J V^2 L_{\mathrm{log}}} \sqrt{ \log \left( \frac{2 J V^2 L_{\mathrm{log}}}{u} \right)  } du \\
&= \sqrt{(s+1)} \int_{0} ^{\infty} (4 J V^2 L_{\mathrm{log}}) y^2 \exp(-y^2) dy.
\end{align*}
Since 
\begin{align*}
\int_{0} ^{\infty} y^2 \exp(-y^2) dy = -\frac{1}{2} \int_{0} ^{\infty} y (-2y \exp(-y^2)) dy = \frac{1}{2} \int_0 ^{\infty} \exp(-y^2) dy = \frac{\sqrt{\pi}}{4},    
\end{align*}
we have $\int_{0} ^{\infty} \sqrt{\log \cN(u, \cH, \norm{\cdot}_{\infty}) } <\infty$ and this concludes the proof.
\end{proof}

\begin{remark}[Different hypothesis spaces]
In essence, the results of Corollary \ref{cor:formal_example_bounds} hold if for a hypothesis space $\cF$, the Rademacher complexity $\mathfrak{R}_n (\cF)$ is $O(n^{-1/2})$ and the entropy integral $\int_{0} ^{\infty} \sqrt{\log \cN(u, \cF, \norm{\cdot}_{\infty}) }$ is bounded.
It is well known that these conditions hold for a reproducing kernel Hilbert space and a linear hypothesis space under mild conditions.
\end{remark}

\begin{remark}[When $\alpha_n$ vanishes fast]
Consider the logistic regression setting, \textit{i.e.}, $P(Y=1 \mid X=x) = \exp(\beta_* ^T x) /(1+ \exp(\beta_* ^T x))$ for some $\beta_* \in \bR^d$. 
\citet[Theorem 1]{blanchet2019} showed that $\cW_p(\bP_{\mathrm{data}}, \bP_n) \leq \frac{1}{\sqrt{n}}$ holds with high probability, under mild conditions on $\bP_{\mathrm{data}}$.
In this case, we choose $\alpha_n =(n^{1/2} \log(n))^{-\frac{1}{p+k}}$.
Then the proposed excess worst-case risk bound is $\cE_{\alpha_n,p} ^{\mathrm{worst}}( \hat{h}_{\alpha_n,p} ^{\mathrm{prop}}) = O_p( n^{-\frac{1+k}{2(p+k)}} \log(n) ^{\frac{p-1}{p+k}} )$. 
By setting $p = \frac{1+k^2}{1-k}$, $\cE_{\alpha_n,p} ^{\mathrm{worst}}( \hat{h}_{\alpha_n,p} ^{\mathrm{prop}}) = O_p( n^{-\frac{1}{2}(1-k)} \log(n) ^{k} )$.
We can choose arbitrary small $k>0$, and thus the convergence rate is near $ O(n^{-1/2})$.\footnote{For $h: \cZ \to \bR$,  $\nabla_z h$ is $(C_{\mathrm{H}},k_1)$-H{\"o}lder continuous, $\sup_{h \in \cH} \sup_{z \in \cZ} \norm{\nabla_z h(z)}_* \leq L$ and any $k_2 \leq k_1$, 
\begin{align*}
\sup_{z, \tilde{z} \in \cZ} \frac{\norm{\nabla_z h(z_1)- \nabla_z h(z_2)}_*}{\norm{z_1-z_2}^{k_2}} &\leq \sup_{\norm{z - \tilde{z}} \leq 1} \frac{\norm{\nabla_z h(z_1)- \nabla_z h(z_2)}_*}{\norm{z_1-z_2}^{k_2}}+ \sup_{\norm{z - \tilde{z}} > 1} \frac{\norm{\nabla_z h(z_1)- \nabla_z h(z_2)}_*}{\norm{z_1-z_2}^{k_2}} \\
&\leq \sup_{\norm{z - \tilde{z}} \leq 1} \frac{\norm{\nabla_z h(z_1)- \nabla_z h(z_2)}_*}{\norm{z_1-z_2}^{k_1}}+ \sup_{\norm{z - \tilde{z}} > 1} \frac{\norm{\nabla_z h(z_1)- \nabla_z h(z_2)}_*}{\norm{z_1-z_2}^{k_2}} \\
&\leq C_{\mathrm{H}} + 2L.
\end{align*}
Thus $\nabla_z h$ is $(C_{\mathrm{H}}+2L,k_2)$-H{\"o}lder continuous.} 
Similar results hold for the excess risk bound.
\end{remark}

\begin{remark}[Regression]
For a constant $B>0$, we let $\cX \times \cY \subseteq [-1,1]^{d-1} \times [-B,B]$ be an open set with respect to the $\ell_2$-norm.
We set $\cZ = \cX \times \cY$ and $\norm{(x,y)} = \sqrt{\norm{x}_{2} ^2  + y ^2}$.
We let $\cH =\{ h(x,y) \mid h(x,y) = |y-f(x)|$ for $f \in \cF_{\mathbf{d}, \mathbf{1}_J, \gamma, s} \}$.\footnote{Since $\nabla_z h(z) = \mathrm{Sign}(y-f(x))[ \nabla_x f(x),  1]^T$, $\bE_{\mathrm{data}}(\norm{\nabla_z h(z)}_2) \geq 1 =: C_{\nabla}.$}
Then similar results hold.
\end{remark}

With the notations defined in the front of this section, we quote the following two lemmas: the Rademacher complexity bound of $\cF_{\mathbf{d}, \mathbf{M}, \gamma}$ by \citet[Corollary 1]{golowich2018} and the covering number bound of $\cF_{\mathbf{d}, \mathbf{1}_{J}, \gamma, s}$ by \citet[Lemma 5]{schmidt2017}.
\begin{lemma}[Rademacher complexity bound]
Assume that $\norm{x}_2 \leq C_{\cX}$. Then
\begin{align*}
\mathfrak{R}_n (\cF_{\mathbf{d}, \mathbf{M}, \gamma}) \leq C_{\cX} \left( \prod_{i=1} ^J M_i \right) \min \left( \bar{\log}^{3/4}(n) \sqrt{\frac{\bar{\log}(\gamma^{-1} \prod_{i=1}^J M_i)}{\sqrt{n}}} , \sqrt{\frac{J}{n}} \right),
\end{align*}
where $\bar{\log}(z) := 1 \vee \log(z)$.
\label{lem:golowich}
\end{lemma}

\begin{lemma}[Covering number bound]
Let $V:=\prod_{i=0} ^{J} (d_i +1)$, then for any $u>0$,
\begin{align*}
\log \cN( u, \cF_{\mathbf{d}, \mathbf{1}_{J}, \gamma, s}, \norm{\cdot}_{\infty}) \leq (s+1) \log \left( \frac{2 J V^2}{u} \right).
\end{align*}
\label{lem:hieber}
\end{lemma}

\newpage
\section{Implementation Details}\label{app:implemenation_details}
In this section, we provide implementation details including the used algorithm and hyper-parameters.
Our algorithm is presented in Algorithm \ref{alg:dro_id}.
Tensorflow implementation for experiments is available at \url{https://github.com/ykwon0407/wdro_local_perturbation}.

\begin{algorithm}[h]
\caption{Principled learning method for WDRO when data are perturbed in classification settings}
\label{alg:dro_id}
\begin{algorithmic}[1]
\STATE {\bfseries Input:} training dataset $\cZ_n = \{(x_1,y_1), \dots, (x_n,y_n)\}$, a (deep neural network) model $f_{\theta}$ parametrized by $\theta$, batch size $B$, hyper-parameters $\tilde{\gamma}_{1},\tilde{\gamma}_{2}, \lambda_{\mathrm{grad}} > 0$, optimization algorithm $\mathfrak{A}$.
\STATE Initialize parameters $\theta$ in $f_{\theta}$
   \WHILE{until a convergent condition is met}
   \STATE Sample $\{(x_{[1]},y_{[1]}), \dots, (x_{[B]},y_{[B]})\} $ from $\cZ_n$
   \FOR{$b = 1$ \textbf{to} $B$}
   \IF{WDRO+MIX}
   \STATE Sample $\gamma$ from $\mathrm{Beta}(\tilde{\gamma}_{1},\tilde{\gamma}_{2})$
   \STATE $x_{[b]}^{\prime} = \gamma x_{[b]}+(1-\gamma) x_{[B+1-b]}$
   \STATE $y_{[b]}^{\prime} = \gamma y_{[b]}+(1-\gamma) y_{[B+1-b]}$\qquad $\rhd \text{  Mixup}$ 
   \ENDIF
   \STATE $h_{\theta}(x_{[b]} ^{\prime}, y_{[b]} ^{\prime}) = \textrm{Cross-entropy loss}\left[ y_{[b]}, f_\theta(x_{[b]}) \right] $ \qquad $\rhd \text{  calculate loss per observation}$ 
   \ENDFOR
   \STATE $\cL = B^{-1} \sum_{b=1} ^B h_{\theta}(x_{[b]} ^{\prime}, y_{[b]} ^{\prime}) + \lambda_{\mathrm{grad}} \norm{\nabla_x h_{\theta}( x_{[b]} ^{\prime}, y_{[b]} ^{\prime} )}_2 ^2$ \qquad $\rhd \text{  calculate the objective function}$  
   \STATE $\theta \leftarrow \mathfrak{A}(\cL, \theta)$ \qquad $\rhd \text{  update parameters}$
   \ENDWHILE
\end{algorithmic}
\end{algorithm}

\subsection{Objective function}
The sample space of the CIFAR-10 and CIFAR-100 datasets can be written as $\cX \times \cY$ where $\cX \subseteq [-1,1]^{3072}$ and $\cY =\{1,\ldots,k\} \subseteq \bR$.
In this space, we define the norm by $\norm{(x,y)} = \norm{x}_2 + 4 \cdot I(y \neq 0)$.
This gives $\norm{\nabla_{z} h(x^{\prime},y^{\prime})}_{*} = \norm{\nabla_{x}h(x^{\prime}, y^{\prime})}_2$ for any $(x^{\prime}, y^{\prime}) \in \cX\times\cY$, as in \eqref{eq:dual_norm_transform}.
Therefore, when $p=p^*=2$, the penalty term in \eqref{eq:prop_perturb_bound} is $\alpha_n \norm{\nabla_z h}_{\bP_n ^{\prime}, p^*}= \alpha_n \sqrt{ n^{-1}\sum_{i=1}^{n} \norm{\nabla_x h(x_i ^{\prime}, y_i ^{\prime})}_2 ^2 }$.
Instead of this term, we use $\lambda_{\rm grad} \left( n^{-1}\sum_{i=1}^{n} \norm{\nabla_x h(x_i ^{\prime}, y_i ^{\prime})}_2 ^2 \right)$ for computational convenience. 

\subsection{Hyper-parameter settings}
We set the penalty parameter $\lambda_{\rm grad} = 0.004$ and the batch size $B=64$.
For MIXUP and WDRO+MIX, the interpolation with hyper-parameters $\tilde{\gamma}_{1}=\tilde{\gamma}_{2}=0.5$ is applied.

For the model architecture, we use the Wide ResNet model with depth 28 and width 2 including the batch normalization and the leaky ReLU activation as in \citet{oliver2018} and \citet{berthelot2019}.
Our implementation of the model and training hyper-parameters closely matches that of \citet{berthelot2019}.

For the optimization algorithm $\mathfrak{A}$, we choose Adam optimizer with the learning rate fixed as $0.002$.
Instead of decaying the learning rate, we use an exponential moving average of the parameters with a decay of $0.999$, and apply a weight decay of $0.02$ at each update for the model as in \citet{berthelot2019}. 
We train the model with $100 \times 2^{16}$ images.

\newpage
\section{Additional experiment: selection of the penalty parameter}
In this section, we compare the accuracy of WDRO and WDRO+MIX with various penalty parameters $\lambda_{\rm grad}$ using the contaminated CIFAR-10 and CIFAR-100 datasets.
The penalty parameters vary as $0.004$, $0.016$, and $0.064$.
The training sample size is 50000 and we apply the salt and pepper noise to 1\% pixels of 10000 test images for the contaminated datasets.
We train the model five times. 

Table \ref{t:regularization} compares accuracy as the penalty parameter changes.
In all cases, a significantly higher accuracy is attained when $\lambda_{\mathrm{grad}} = 0.016$ than other $\lambda_{\mathrm{grad}}$ values.
With this result, we anticipate that our proposed methods can achieve higher accuracy than the one in Section \ref{s:numerical_experiments}, by carefully selecting the penalty parameter $\lambda_{\rm grad}$.

\begin{table}[!ht]
\caption{Accuracy comparison WDRO and WDRO+MIX with various penalty parameter $\lambda_{\mathrm{grad}}$. Other details are given in Table \ref{t:compare}.}
\label{t:regularization}
\vskip 0.15in
\begin{center}
\begin{small}
\begin{sc}
\begin{tabular}{lccccc}
\toprule
\multirow{2}{*}{METHODS} & \multicolumn{3}{c}{$\lambda_{\mathrm{grad}}$} \\
\cmidrule{2-4}
& 0.004 & 0.016 & 0.064 \tabularnewline
\midrule
CIFAR-10 & & & \tabularnewline
WDRO  & $87.4\pm0.4$ & $\mathbf{87.9\pm0.2}$ & $86.2\pm0.2$ \tabularnewline
WDRO+MIX & $87.3\pm0.4$ & $\mathbf{88.2\pm0.3}$ & $86.8\pm0.2$ \tabularnewline
\midrule
CIFAR-100 & & & \tabularnewline
WDRO & $62.1\pm0.4$ & $\mathbf{64.1\pm0.3}$ & $62.6\pm0.4$ \tabularnewline
WDRO+MIX & $60.6\pm0.7$ & $\mathbf{62.2\pm 0.2}$ & $61.3\pm0.2$ \tabularnewline
\bottomrule
\end{tabular}
\end{sc}
\end{small}
\end{center}
\vskip -0.1in
\end{table}

\end{document}